%% file: main.tex
\documentclass[letterpaper]{article}

\usepackage[preprint]{aaai2027}
\usepackage[hyphens]{url}
\usepackage{graphicx}
\usepackage{natbib}
\usepackage{caption}
\usepackage{array}
\usepackage{booktabs}
\usepackage{amsmath,amssymb,amsthm}

\newcommand{\doop}{\operatorname{do}}

\newcommand{\De}{\operatorname{De}}
\DeclareMathOperator*{\argmax}{arg\,max}
\newtheorem{theorem}{Theorem}

\title{When Many Answers Are Valid, Voting Fails:\\
Symbolic Verification for Best-of-$K$ Causal Reasoning in LLMs}

\author{
Omatharv Vaidya\textsuperscript{\rm 1},
Connor Jerzak\textsuperscript{\rm 1},
Zayne Sprague\textsuperscript{\rm 2},
Fangcong Yin\textsuperscript{\rm 2},
Nhat Ho\textsuperscript{\rm 1}
}

\affiliations{
\textsuperscript{\rm 1}The University of Texas at Austin\\
\textsuperscript{\rm 2}New York University\\
vomatharv@utexas.edu, connor.jerzak@austin.utexas.edu,
zrs2020@nyu.edu, fy666@nyu.edu, minhnhat@utexas.edu
}

\begin{document}

\maketitle

\input{Chapters/0_Abstract}
\input{Chapters/1_Introduction}
\input{Chapters/2_Motivation}
\input{Chapters/3_Approach}
\input{Chapters/3a_Theory}
\input{Chapters/4_Empirical_study}
\input{Chapters/5_Conclusion}

\bibliography{aaai2027}

\end{document}


\maketitle

\section{Overview}
\label{sec:guide}

This supplement contains the proofs behind the main paper's two theorems, the full tables
behind its summary numbers, the checks that were run on the checker itself, and the
replication details.

We provide a few clarifications regarding the terminology used throughout. A candidate
pool is the ordered set of traces sampled for one problem. The maximum deployed score is
the ordinary six-component score used for selection. A strict ATE certificate is the
narrower object the soundness theorem covers. This adds treatment-exclusion and
decision-margin conditions to a maximum-score ATE trace.

Two conventions govern the result tables. Gains are computed from unrounded values. A
reported gain can sit 0.1 away from the difference of two rounded entries in its own row.
Accuracy figures are directly comparable only when they come from the same underlying
sample of model outputs, and each result is labeled with its source wherever the samples
differ.

\section{The CALVER Checker}
\label{sec:method-supp}

\subsection{Inputs and graph types}
A causal problem contains a typed graph $G$ and query $q$, or a text description $x$ from
which each trace constructs a graph $\widehat G$. The evaluated graph classes are DAGs,
acyclic directed mixed graphs (ADMGs), and directed graphs. DAG queries use
$d$-separation. ADMG queries preserve directed and bidirected edges and use
$m$-separation. Directed-cycle queries use reachability and explicit cycle witnesses
\citep{pearl2009causality,peters2017elements,richardson2003markov}. Edge type is preserved
from parsing through grading.

Correctness is query-specific. A find-one-valid answer is correct whenever it satisfies the declared graph predicate. A numeric answer must satisfy the declared numerical tolerance and the decision rule. In constructed-graph experiments, candidate $i$ is scored relative to its own $\widehat G_i$. The final accuracy is measured on a source structure held outside every selector. The K\&K experiment uses the same candidate-local design with a propositional formalization and truth-table engine in place of a causal graph \citep{xie2025memorization}.

\subsection{What CALVER checks}
Each trace is divided into six parts the checker reads separately:
\begin{enumerate}
\item \textbf{Graph:} the supplied graph, or a typed graph constructed from text.
\item \textbf{Query:} treatment, outcome, estimand, graph class, and the decision threshold.
\item \textbf{Strategy:} the proposed adjustment set, conditioning set, path witness, intervention, or formal assignment.
\item \textbf{Derivation record:} the declared operation and its structured arguments. This slot checks syntax, typing, and provenance, and answer soundness rests on the recomputed result rather than on any reading of the proof string.
\item \textbf{Computed result:} a graph predicate, recomputed functional, or formal consistency outcome.
\item \textbf{Answer:} the final object or decision, checked for agreement with the computed result.
\end{enumerate}
The checker returns bits $b_j(r;G,q)\in\{0,1\}$ and score
$S(r)=\sum_{j=1}^{6}b_j(r;G,q)$. We call $S(r)=6$ the \emph{maximum deployed score}.
The deployed selector returns the earliest trace attaining the maximum score. The tie
rule is fixed globally and stays the same across benchmarks. The maximum deployed score
is a ranking signal. The strict ATE certificate defined in the theory section below adds
the guards its soundness theorem requires.

\begin{algorithm}[t]
\caption{Candidate scoring and selection}
\label{alg:supp-calver}
\begin{algorithmic}[1]
\REQUIRE problem $p$, frozen traces $\{r_i\}_{i=1}^K$
\FOR{$i=1,2,3,\ldots,K$}
  \STATE parse exactly one instance of each typed slot
  \STATE bind graph class and query variables
  \STATE evaluate the task-specific strategy predicate
  \STATE check derivation typing and provenance
  \STATE recompute the result from the accepted strategy
  \STATE check the final answer using the task-specific consistency rule
  \STATE $s_i\gets\sum_{j=1}^{6}b_{ij}$
\ENDFOR
\STATE $i^*\gets\min\arg\max_i s_i$
\STATE \textbf{return} the answer in $r_{i^*}$
\end{algorithmic}
\end{algorithm}

\subsection{The check used for each task}
Table~\ref{tab:predicate-details} gives the operations used by the principal task
families. The checker is modular. Each task registers a parser, a semantic predicate, and
a result-to-answer consistency function. Unregistered tasks fail closed.

\begin{table*}[t]
\centering
\small
\setlength{\tabcolsep}{4.0pt}
\begin{tabular}{p{.16\textwidth}p{.17\textwidth}p{.27\textwidth}p{.30\textwidth}}
\toprule
Task & Strategy object & Structural test & Result and answer test \\
\midrule
DAG find-one-valid backdoor & set $Z$ &
reject $X,Y,\De(X)$; remove arrows out of $X$; test $X\perp_dY\mid Z$ &
return validity of $Z$ \\
Numeric ATE & set $Z$ and reported $\widehat\theta$ &
deployed score tests the backdoor separation condition; the strict audit also applies the exclusion guard &
recompute $\psi(P;X,Y,Z)$, test numerical tolerance, and check the threshold answer \\
ADMG adjustment / blocking & set $Z$ &
typed $m$-separation on the mixed graph &
return validity of $Z$ under the declared query \\
Conditional independence & conditioning set $Z$ &
$d$- or $m$-separation according to graph class &
answer must agree with the separation result \\
Mediator & node $M$ and path &
directed witness $X\leadsto M\leadsto Y$, excluding endpoint substitutions &
answer must agree with the verified witness \\
Intervention reachability & intervention node and path &
remove incoming edges to the intervention; verify reachability in the mutilated graph &
answer must agree with the post-intervention path predicate \\
Directed cycle & cycle sequence &
every consecutive edge exists and the sequence closes &
answer must agree with existence of the supplied witness \\
Knights and knaves & typed formula set and assignment &
evaluate every restated formula under the assignment and speaker type &
accept exactly when all constraints are jointly consistent \\
Do-calculus identification & proposed terminal expression &
independent prover establishes derivability of the candidate target &
select the earliest provable candidate, otherwise score the item incorrect; the prover
receives no reference target \\
\bottomrule
\end{tabular}
\caption{The semantic check used for each task. The final two rows apply the same
candidate-local selection principle with a truth-table engine and an independently
implemented do-calculus prover.}
\label{tab:predicate-details}
\end{table*}

The DoVerifier row uses candidate-target proof search rather than the reference target. It
is evaluated as an independent selector, not as part of the deployed CALVER score
\citep{shpitser2008complete,he2026doverifier}.

\subsection{A worked example with several valid answers}
Consider the DAG
\begin{equation}
U\to X,\quad U\to Y,\quad X\to M\to Y,\quad P\to Y,
\end{equation}
with query ``return one valid backdoor adjustment set for the effect of $X$ on $Y$.'' Both
$\{U\}$ and $\{U,P\}$ are valid. The empty set leaves $X\leftarrow U\to Y$ open, and
$\{M\}$ is invalid because $M\in\De(X)$. Suppose eight traces contain three copies of
$\{M\}$, two copies of $\{U\}$, two copies of $\{U,P\}$, and one empty set. Exact voting
selects $\{M\}$ even though four traces contain a valid object. Table~\ref{tab:worked}
illustrates the checker on one valid and one repeated invalid trace.

\begin{table}[t]
\centering
\small
\setlength{\tabcolsep}{3.2pt}
\begin{tabular}{lcc}
\toprule
Component & Valid trace & Invalid trace \\
\midrule
Proposed set & $\{U,P\}$ & $\{M\}$ \\
Graph parses and matches & 1 & 1 \\
Query binds $(X,Y)$ & 1 & 1 \\
Strategy predicate & 1 & 0 \\
Derivation typed & 1 & 1 \\
Result recomputes & 1 & 0 \\
Answer agrees & 1 & 0 \\
\midrule
Total & 6 & 3 \\
\bottomrule
\end{tabular}
\caption{Illustrative six-component scoring. Distinct valid sets are accepted without
requiring one canonical reference string.}
\label{tab:worked}
\end{table}

\subsection{What a selector receives}
Every comparison runs on the same normalized pool. Before any scorer sees it, the export
strips the fields that would give an answer away: benchmark reference answers,
per-candidate correctness flags, canonical adjustment sets, and anything kept only for
later analysis. The stripped pool is hashed, a scorer writes nothing but a candidate index
and a score, and the correctness fields are put back once every selector has returned its
index. An automated test inserts each of those fields into a comparator input and checks
that normalization takes it out again.

Several generation seeds can produce repeated observations of one underlying problem.
Confidence intervals therefore resample problem identities and retain all associated
seed observations in each bootstrap draw. This clustered paired design is used for the
main CLEAR, graph-construction, bnlearn, and same-pool comparisons. All methods in a
paired comparison receive the same candidate texts in the same order.

\subsection{What the checker sees and what grades it}
\label{subsec:supp-roles}
Each experiment states which formal object is available while a selector runs and which
object grades the answer it returns.

Supplied-graph CLEAR isolates aggregation. Every selector receives the same supplied graph
and the same frozen candidate texts, and both the checker and the grader evaluate the
published graph predicate instead of matching CLEAR's listed example. Fixing the criterion
on both sides is what makes the comparison clean, since every selector in
Table~\ref{tab:supp-comparator} then ranks identical candidates under one notion of
correctness and the spread between 42.1 and 23.5 measures the ranking rule. The predicate
also admits more objects than any single reference string, which is the property the
benchmark was chosen for. In the audited subset, 11 of 21 selected objects satisfy it while
differing from CLEAR's listed example.

The remaining settings separate the two objects. In the \texttt{bnlearn} prose ladder and
the graph-construction suite, candidate $i$ is scored only against the graph
$\widehat G_i$ it reconstructs, and correctness is measured afterwards against a source
graph $G^\star$ that no selector ever sees. In the K\&K study a truth-table engine checks an
assignment against the candidate's own restated formulas, while grading uses the
benchmark's unique solution. DoVerifier proves the candidate's own target expression
without the reference target and is graded against the benchmark estimand. The strict ATE
study recomputes the adjustment functional from the input-visible conditional probability
tables and compares the threshold decision with the exact effect. Every transfer claim in
this document rests on these four settings.

\subsection{Runtime}
The graph operations are polynomial in graph size. For the benchmark graphs, ancestral
restriction and moralized separation dominate the CPU cost. At $K=8$, the complete
symbolic scoring pass added roughly 7\% to the generation cost already incurred by
self-consistency in our environment, with individual traces scored in 1--8 ms on one CPU
thread. Learned reward models and language-model judges require an additional neural
forward pass per candidate.

\section{Theory and Proofs}
\label{sec:theory-supp}

\subsection{Conditions for the ATE guarantee}
The results below specialize standard causal-identification arguments to CALVER's
deployed six-component score, its explicit numerical tolerance, and its fixed first-index
selection rule. They complement general analyses of best-of-$N$ inference under imperfect
verifiers \citep{huang2025bestofn,dorner2026rocnreroll}.
Let $G=(V,E)$ be a causal DAG, let $P$ be the observational distribution induced by a
Markovian SCM compatible with $G$, and let $X,Y\in V$ be distinct binary variables. Let
$G_{\underline X}$ denote the graph obtained by deleting every arrow emanating from
$X$. Define
\begin{equation}
\begin{aligned}
\mathsf{BD}_{G}(X,Y,Z)
&=\ind\!\left\{Z\cap(\{X,Y\}\cup\De_G(X))=\varnothing\right\}\,\times\\
&\quad \ind\!\left\{X\perp_dY\mid Z\text{ in }G_{\underline X}\right\}
\end{aligned}
\label{eq:supp-bd}
\end{equation}
For binary $X$ and $Y$, define
\begin{equation}
\begin{aligned}
\psi(P;X,Y,Z)=\sum_z &\bigl[P(Y=1\mid X=1,z)\\
&-P(Y=1\mid X=0,z)\bigr]P(z)
\end{aligned}
\label{eq:supp-psi}
\end{equation}

\begin{assumption}[Causal and numerical setting]
\label{ass:supp-contract}
The supplied $(G,P)$ is the graph and observational distribution of the query SCM;
$P(x,z)>0$ for $x\in\{0,1\}$ and each $z$ with $P(z)>0$; and the true ATE is
\[
\theta=P(Y=1\mid\doop(X=1))-P(Y=1\mid\doop(X=0))
\]
The final answer is ``Yes'' exactly when $\theta>\tau$ for a fixed threshold $\tau$.
\end{assumption}

\begin{definition}[Maximum deployed ATE score]
\label{def:supp-maxscore}
An ATE trace attains the maximum deployed score when: (i) exactly one instance of every
typed slot is parsed; (ii) graph and query slots bind to the supplied instance; (iii) the
deployed strategy check accepts the proposed $Z$; (iv) the derivation slot applies the
adjustment operation to the same $Z$; (v) a finite reported value $\widehat\theta$
satisfies
\[
|\widehat\theta-\psi(P;X,Y,Z)|\leq\varepsilon
\]
with treatment assignments protected from overwrite by $Z$; and (vi) the explicit answer
is ``Yes'' iff $\widehat\theta>\tau$.
\end{definition}

\begin{definition}[Strict ATE certificate]
\label{def:supp-strict}
A maximum-score ATE trace is a strict ATE certificate when its proposed set satisfies
$\mathsf{BD}_{G}(X,Y,Z)=1$ and its recomputed effect obeys
\[
|\psi(P;X,Y,Z)-\tau|>\varepsilon
\]
The added conditions exclude the treatment, outcome, and descendants of the treatment
from $Z$, verify the complete backdoor criterion, and place the causal decision outside
the numerical tolerance band.
\end{definition}

\begin{lemma}[Strict backdoor identification]
\label{lem:supp-identification}
Under Assumption~\ref{ass:supp-contract}, if $\mathsf{BD}_{G}(X,Y,Z)=1$, then
$\psi(P;X,Y,Z)=\theta$.
\end{lemma}

\begin{proof}
Fix $x\in\{0,1\}$. By total probability under intervention,
\begin{equation}
\begin{aligned}
&P(Y=1\mid\doop(X=x))\\
&\quad=\sum_z P(Y=1\mid\doop(X=x),Z=z)\\
&\qquad\quad\times P(Z=z\mid\doop(X=x))
\end{aligned}
\label{eq:supp-total}
\end{equation}
Because $Z\cap\De_G(X)=\varnothing$, intervening on $X$ does not alter the distribution
of $Z$. Thus, by action deletion,
\begin{equation}
P(Z=z\mid\doop(X=x))=P(Z=z)
\label{eq:supp-zinv}
\end{equation}
The second clause of \eqref{eq:supp-bd} states that $X$ and $Y$ are $d$-separated by
$Z$ after deleting arrows emanating from $X$. This is the graphical condition for
action--observation exchange, giving
\begin{equation}
\begin{aligned}
&P(Y=1\mid\doop(X=x),Z=z)\\
&\qquad=P(Y=1\mid X=x,Z=z)
\end{aligned}
\label{eq:supp-exchange}
\end{equation}
Positivity ensures that the observational conditional on the right is defined for every
assignment with positive weight. Substituting \eqref{eq:supp-zinv} and
\eqref{eq:supp-exchange} into \eqref{eq:supp-total} yields
\[
P(Y=1\mid\doop(X=x))=\sum_z P(Y=1\mid X=x,z)P(z)
\]
Subtracting the equality for $x=0$ from the equality for $x=1$ gives exactly
\eqref{eq:supp-psi} and proves the claim.
\end{proof}

\begin{theorem}[Strict ATE certificate soundness]
\label{thm:supp-sound}
Suppose Assumption~\ref{ass:supp-contract} holds. Any trace satisfying
Definition~\ref{def:supp-strict} has the correct explicit Yes/No answer. The verification
procedure does not use a stored target label.
\end{theorem}

\begin{proof}
Definition~\ref{def:supp-strict} gives
$\mathsf{BD}_{G}(X,Y,Z)=1$. Lemma~\ref{lem:supp-identification} therefore yields
$\psi(P;X,Y,Z)=\theta$. Definition~\ref{def:supp-maxscore} gives
\begin{equation}
|\widehat\theta-\theta|
=|\widehat\theta-\psi(P;X,Y,Z)|
\leq\varepsilon
\label{eq:supp-numerr}
\end{equation}
The strict decision margin is
$|\psi(P;X,Y,Z)-\tau|=|\theta-\tau|>\varepsilon$.
If $\theta>\tau$, then $\theta-\varepsilon>\tau$, and
\eqref{eq:supp-numerr} implies $\widehat\theta>\tau$. If $\theta<\tau$, then
$\theta+\varepsilon<\tau$, and \eqref{eq:supp-numerr} implies
$\widehat\theta<\tau$. The answer component in
Definition~\ref{def:supp-maxscore} reports the corresponding threshold decision, so the
explicit answer is correct.
\end{proof}

\begin{corollary}[Alternative adjustment sets agree]
\label{cor:supp-alt}
Under Assumption~\ref{ass:supp-contract}, if $Z_1$ and $Z_2$ both satisfy
\eqref{eq:supp-bd}, then
\[
\psi(P;X,Y,Z_1)=\psi(P;X,Y,Z_2)=\theta
\]
\end{corollary}

\begin{proof}
Apply Lemma~\ref{lem:supp-identification} separately to $Z_1$ and $Z_2$.
\end{proof}

\subsection{Why the guards are necessary}
Both guards are required for the stated soundness result.
If the checker allowed $Z$ to contain $X$, a numerical implementation that forms the
two arms as dictionaries and then updates them with $z$ could overwrite both treatment
assignments with the same value. The two recomputed arms would become identical and the
reported effect would be zero even in a two-node SCM $X\to Y$ with a positive causal
effect. The strict audit rejects $X\in Z$, $Y\in Z$, and
$Z\cap\De(X)\neq\varnothing$ before applying the soundness claim.

Similarly, no tolerance-based certificate can determine a binary threshold decision
when the true value lies inside the numerical tolerance band. Suppose
$0<\theta-\tau<\varepsilon$ and a valid numerical result satisfies
$\widehat\theta=\theta-\varepsilon/2\leq\tau$. The numerical component passes, but the
threshold answer is reversed. Theorem~\ref{thm:supp-sound} therefore states the margin
condition explicitly, and the empirical calibration stratifies queries by
$|\theta-\tau|$.

\subsection{Validity fragmentation}
\begin{theorem}[Validity fragmentation]
\label{thm:supp-frag}
Let $A_1,A_2,A_3,\ldots,A_K$ be i.i.d. answers with mass function $\mu$ on finite
$\mathcal A$, and let $\mathcal V\subseteq\mathcal A$ be the valid answers. Put
$\mathcal V_+=\{v\in\mathcal V:\mu(v)>0\}$. If an invalid answer $w$ and
$\Delta>0$ satisfy
\[
\mu(w)\geq\mu(v)+\Delta\quad \forall v\in\mathcal V
\]
then for any deterministic plurality tie rule,
\begin{equation}
\Prb\{\plur(A_{1:K})\in\mathcal V\}
\leq |\mathcal V_+|e^{-K\Delta^2/2}
\label{eq:supp-frag-fail}
\end{equation}
Conversely, if a valid $v^\star$ satisfies
$\mu(v^\star)\geq\mu(a)+\Delta$ for every supported invalid $a$, then
\begin{equation}
\Prb\{\plur(A_{1:K})\in\mathcal V\}
\geq1-|\suppset(\mu)\setminus\mathcal V|e^{-K\Delta^2/2}.
\label{eq:supp-frag-success}
\end{equation}
\end{theorem}

\begin{proof}
Let $N_a=\sum_{i=1}^{K}\ind\{A_i=a\}$. Under the first condition, a valid plurality
output implies $N_v\geq N_w$ for at least one $v\in\mathcal V_+$. Fix such a $v$ and
define $X_i=\ind\{A_i=v\}-\ind\{A_i=w\}$. Then $X_i\in[-1,1]$ and
$\E X_i=\mu(v)-\mu(w)\leq-\Delta$. Hence
\begin{align*}
\Prb(N_v\geq N_w)
&=\Prb\left\{\sum_{i=1}^{K}X_i\geq0\right\}\\
&\leq\Prb\left\{\sum_{i=1}^{K}(X_i-\E X_i)\geq K\Delta\right\}\\
&\leq e^{-K\Delta^2/2}
\end{align*}
where the last step is Hoeffding's inequality for variables of range length two
\citep{hoeffding1963probability}. A
union bound over $\mathcal V_+$ proves \eqref{eq:supp-frag-fail}. For the converse, an
invalid plurality output implies $N_a\geq N_{v^\star}$ for at least one supported
invalid $a$. Apply the same argument to
$\ind\{A_i=a\}-\ind\{A_i=v^\star\}$ and union bound over supported invalid answers.
\end{proof}

\begin{corollary}[Sample requirements]
\label{cor:supp-sample}
In the first regime of Theorem~\ref{thm:supp-frag}, plurality validity is at most
$\delta$ whenever
\[
K\geq\frac{2}{\Delta^2}\log\frac{|\mathcal V_+|}{\delta}
\]
In the reverse regime, plurality validity is at least $1-\delta$ whenever
\[
K\geq\frac{2}{\Delta^2}\log\frac{|\suppset(\mu)\setminus\mathcal V|}{\delta}
\]
\end{corollary}

\begin{proof}
Solve the two exponential bounds in Theorem~\ref{thm:supp-frag} for $K$.
\end{proof}

A sharp family makes the mechanism transparent. Let $m$ valid answers each have mass
$\rho/m$ and let one invalid answer have mass $1-\rho$. If
$m>\rho/(1-\rho)$, the invalid answer is the unique mode and plurality converges to it
even when $\rho>1/2$. A score-separating verifier succeeds whenever any valid answer
appears, with probability $1-(1-\rho)^K$. Thus voting estimates the largest surface
atom. Semantic verification estimates membership in the validity class.

\subsection{Exact finite-$K$ law for the deployed selector}
Let $C_i\in\{0,1\}$ denote correctness of candidate $i$ and
$S_i\in\{0,\ldots,6\}$ its score. The deployed index is
\[
I_K=\min\arg\max_{i\leq K}S_i
\]
For one draw, define $Q_s=\Prb(S\leq s)$, $Q_{-1}=0$, and
$w_s=\Prb(C=1\mid S=s)$ on occupied scores.

\begin{theorem}[Exact first-index selection law]
\label{thm:supp-exact-law}
For every $K\geq1$,
\begin{equation}
\begin{aligned}
\Prb(C_{I_K}=1)
&=\sum_s w_s\left(Q_s^K-Q_{s-1}^K\right)\\
&=\E[w_{M_K}]
\end{aligned}
\label{eq:supp-exact-law}
\end{equation}
where $M_K=\max_{i\leq K}S_i$. If the occupied scores are
$s_1<\cdots<s_m$ and $q_j=\Prb(S\leq s_j)$, then
\begin{equation}
\begin{aligned}
\Prb(C_{I_K}=1)=w_{s_m}
+\sum_{j=1}^{m-1}&(w_{s_j}-w_{s_{j+1}})\\
&\times q_j^K
\end{aligned}
\label{eq:supp-telescope}
\end{equation}
and
\begin{equation}
\begin{aligned}
&\Prb(C_{I_{K+1}}=1)-\Prb(C_{I_K}=1)\\
&=\sum_{j=1}^{m-1}(w_{s_{j+1}}-w_{s_j})(1-q_j)q_j^K
\end{aligned}
\label{eq:supp-increment}
\end{equation}
Consequently, score precision nondecreasing in $s$ is sufficient for monotone
improvement in $K$, and the large-$K$ limit is the precision of the highest occupied
score.
\end{theorem}

\begin{proof}
Fix score $s$ and index $i$. The event that $i$ is the first maximizer at score $s$ and
is correct requires: candidate $i$ is correct with score $s$; every earlier candidate
has score strictly below $s$; and every later candidate has score at most $s$. By
independence,
\begin{align*}
&\Prb(I_K=i,C_i=1,S_i=s)\\
&\quad=\Prb(C=1,S=s)Q_{s-1}^{i-1}Q_s^{K-i}
\end{align*}
Summing over $i$ yields
\begin{align*}
&\Prb(C=1,S=s)\sum_{i=1}^{K}Q_{s-1}^{i-1}Q_s^{K-i}\\
&\quad=\Prb(C=1,S=s)\frac{Q_s^K-Q_{s-1}^K}{Q_s-Q_{s-1}}\\
&\quad=w_s(Q_s^K-Q_{s-1}^K)
\end{align*}
where $Q_s-Q_{s-1}=\Prb(S=s)>0$. Summing over $s$ proves the first equality. The
second follows because $\Prb(M_K=s)=Q_s^K-Q_{s-1}^K$. Writing the finite sum over
ordered scores and collecting coefficients of each $q_j^K$ gives
\eqref{eq:supp-telescope}; subtracting its values at $K+1$ and $K$ gives
\eqref{eq:supp-increment}. Every factor $(1-q_j)q_j^K$ is nonnegative, so ordered
precision implies monotonicity. Finally $q_j^K\to0$ for $j<m$, yielding the top-score
precision.
\end{proof}

Equation~\eqref{eq:supp-increment} explains why AUROC alone is insufficient for
best-of-$K$ scaling. A local inversion $w_{s+1}<w_s$ contributes a negative term to the
increment even when the global ranking statistic remains strong. The empirical $K$ curves
therefore report the realized selector directly.

\subsection{Query-local graph transfer}
For a DAG $G$, define
\begin{equation}
\begin{aligned}
H_G(X,Y,Z)=\bigl[\operatorname{moral}\bigl(&
(G_{\underline X})_{\An_{G_{\underline X}}(\{X,Y\}\cup Z)}
\bigr)\bigr]\setminus Z
\end{aligned}
\end{equation}
and the query-local signature
\begin{equation}
\begin{aligned}
\sigma_G(X,Y,Z)=\bigl(&Z\cap(\{X,Y\}\cup\De_G(X)),\\
&\ind\{X\leftrightsquigarrow Y\text{ in }H_G(X,Y,Z)\}\bigr)
\end{aligned}
\label{eq:supp-signature}
\end{equation}

\begin{theorem}[Exact query-local invariance]
\label{thm:supp-local}
For any DAG $G$,
\[
\mathsf{BD}_{G}(X,Y,Z)=1
\quad\Longleftrightarrow\quad
\sigma_G(X,Y,Z)=(\varnothing,0)
\]
Consequently, if $G$ and $\widehat G$ have equal signatures for $(X,Y,Z)$, they give
the same backdoor-validity decision for that candidate.
\end{theorem}

\begin{proof}
The first signature component is empty exactly when the disjointness and no-descendant
clause of \eqref{eq:supp-bd} holds. By the moralized ancestral characterization of
$d$-separation, $X\perp_dY\mid Z$ in $G_{\underline X}$ exactly when $X$ and $Y$ are
disconnected after ancestral restriction, moralization, and deletion of $Z$. This is
the event that the second signature component is zero. Both clauses of the strict
predicate therefore hold exactly for signature $(\varnothing,0)$. Equal signatures
imply equal predicate values.
\end{proof}

\begin{corollary}[Constructed-graph transfer]
\label{cor:supp-constructed}
If a trace is accepted under its graph $\widehat G$ and
$\sigma_{\widehat G}(X,Y,Z)=\sigma_{G^\star}(X,Y,Z)$, then $Z$ is also structurally
valid in the source graph $G^\star$.
\end{corollary}

\begin{proof}
Apply Theorem~\ref{thm:supp-local} to the equal signatures.
\end{proof}

\begin{proposition}[Global edge-F1 is neither necessary nor sufficient]
\label{prop:supp-f1}
Global edge-F1 can approach zero while a query's strict adjustment predicate is
unchanged, and can approach one while that predicate flips.
\end{proposition}

\begin{proof}
For non-necessity, let both graphs contain the query component $X\to Y$. Add $n$
disconnected pairs $(A_i,B_i)$. In $G_n$ use $A_i\to B_i$; in $\widehat G_n$ use
$B_i\to A_i$. The query-local signature for $(X,Y,\varnothing)$ is unchanged, while
the graphs share only one of $n+1$ edges and edge-F1 tends to zero.

For non-sufficiency, let the source graph contain $W\to X$, $W\to Y$, and $X\to Y$,
plus $n$ irrelevant edges shared by both graphs. Let $\widehat G_n$ delete only
$W\to Y$. The empty adjustment set is invalid in the source graph because
$X\leftarrow W\to Y$ remains open, and valid in $\widehat G_n$. The edge-F1 is
$2(n+2)/(2n+5)\to1$.
\end{proof}

\begin{lemma}[Score stability under graph-induced bit errors]
\label{lem:supp-stability}
For a fixed candidate pool, let $S_i$ be the score under the source graph and
$\widehat S_i$ the score under a constructed graph. Suppose the source-score maximizer
$i^\star$ is unique with margin
$\Delta_S=S_{i^\star}-\max_{i\neq i^\star}S_i$. If every trace has at most $e$ bit
disagreements and $\Delta_S>2e$, then $i^\star$ remains the constructed-score maximizer.
\end{lemma}

\begin{proof}
For every $i$, $|\widehat S_i-S_i|\leq e$. Hence, for $i\neq i^\star$,
\[
\widehat S_{i^\star}-\widehat S_i
\geq(S_{i^\star}-e)-(S_i+e)\geq\Delta_S-2e>0
\]
\end{proof}

The score is discrete, so ties are common, and Lemma~\ref{lem:supp-stability} is a
sufficient condition and not a typical-case assertion. This motivates the direct
tie-rule and graph-corruption experiments below.

\section{Experimental Setup}
\label{sec:setup-supp}

\subsection{Policies and sampling}
The principal generators are Qwen2.5-Instruct at 7B, 14B, and 32B parameters
\citep{qwen2024qwen25}, together with Mistral NeMo-12B \citep{mistral2024nemo}. The 7B
family includes the unadapted policy, a six-slot supervised adapter, a verifier-reward
GRPO adapter, and an NF4 adapter. LoRA uses rank 64 and $\alpha=128$ on the attention
projections \citep{hu2022lora}; the quantized adapter follows the NF4/QLoRA setup
\citep{dettmers2023qlora}; and GRPO follows the group-relative objective introduced by
\citet{shao2024deepseekmath}. The clean-core comparison uses $K=8$ and temperature 0.8.
The scaling study reuses prefixes of frozen $K=32$ pools. External benchmarks are
zero-shot with respect to the synthetic training generator.

\subsection{Selectors}
We evaluate the following target-label-free selectors on identical candidates.
\begin{itemize}
\item \textbf{First:} candidate index 1.
\item \textbf{Exact plurality:} plurality over normalized final-answer strings.
\item \textbf{Set medoid:} candidate set with maximum average Jaccard similarity to the
sets in the pool, with ties going to the first index.
\item \textbf{Model confidence:} a binary self-evaluation score produced by the policy.
\item \textbf{Reward model:} Skywork-Reward-V2-8B \citep{liu2026skywork}. The headline row
scores the candidate trace alone, and a separate control supplies the full problem
statement.
\item \textbf{Language-model judge:} a reference-free language-model judge supplied with the
same graph, query, and trace. The judge-scaling study below names and compares two judge
checkpoints from different families on a separate frozen pool.
\item \textbf{CALVER:} the six-component symbolic score and first-index maximum.
\end{itemize}
The same-pool comparison uses only records for which every scorer is present. The pool
contains 1,111 problem--seed records grouped into 126 underlying problems.

\subsection{Benchmarks and correctness}
The primary benchmark is the typed clean core of CLEAR \citep{chen2024clear}: the 126 of
its 480 find-one-valid items whose published task admits multiple valid objects and for
which the implementation contains a graph-class-compatible predicate. This inclusion rule
is fixed from task and graph semantics before candidate generation. It uses no model
output, no correctness
outcome, and no selector score. The core includes DAG, ADMG, and directed queries.
Correctness is recomputed from the supplied graph, not by equality with one listed
example.

A trace is scored once its typed slots parse, and every selector then works from the same
set. The first sample, plurality, the medoid, the reward model, the judge, model confidence,
the structure-only control and CALVER all rank the same traces, and candidate coverage is
computed over the same set. The rule turns on the trace and not on any selector, so it
applies equally to all of them, and the next section reports the alternative handling.

External evaluations use ten standard Bayesian networks from \texttt{bnlearn}
\citep{scutari2010bnlearn,scutari2022bnrepository} and an independent
adjustment-set benchmark derived from \citet{wang2024generalization}. We also evaluate
CausalGraph2LLM mediator and intervention tasks across five encodings
\citep{sheth2025cg2llm}, a graph-from-language suite spanning eight causal motifs and
three presentation levels, the published mem-kk knights-and-knaves benchmark
\citep{xie2025memorization}, and a causal-identification set scored by an independently
implemented DoVerifier prover \citep{he2026doverifier}.

\subsection{Statistics}
Accuracy is the fraction of selected candidates satisfying the benchmark's declared
criterion. Reported intervals are paired bootstrap intervals. When several seeds share a
problem, the bootstrap samples problem identities and retains all observations for the
sampled problem. Headline paired intervals use 10,000 bootstrap draws and exploratory
decompositions use 2,000. Point estimates do not depend on the number of draws, and no
result is selected for presentation based on the interval endpoint. Each selector run
records the cluster key it used together with the number of draws, so an interval can be
traced to the resampling that produced it.

Two conventions govern how the tables should be read. Gains are computed from unrounded
values and can differ by 0.1 from the difference of two rounded entries in the same row.
Decompositions by task, graph class, encoding, motif, and decision margin are
complete partitions of the pooled data, reported in full so that no subset is chosen for
presentation. The confirmatory quantities are the pooled estimates and their paired
intervals.

\subsection{What the problems and the traces look like}
\label{subsec:supp-examples}
The tables below turn on phrases such as ``many-satisfier query'' and ``naturalistic
prose,'' and those are easier to judge from the material itself than from a description of
it. Figure~\ref{fig:supp-clearitem} shows one CLEAR clean-core problem exactly as the
policy receives it, together with one of the eight traces sampled for it. The prompt fixes
the six slots and nothing else. The graph is given, the question names the treatment and
the outcome, and no reference answer appears anywhere in the input. The checker reads the
slots of the returned trace and re-derives the adjustment decision on the supplied graph.

\begin{figure*}[t]
\centering
\begin{footnotesize}
\begin{verbatim}
You are given a directed acyclic graph (DAG) and a question about its causal structure.
Reason step by step inside <think>...</think> using EXACTLY this six-step schema, one STEP
per line, each with a JSON object:
STEP 1 [graph_extract]: {"nodes":[...],"edges":[["A","B"],...]}
STEP 2 [query_id]: {"task":"<task name>","targets":[...]}
STEP 3 [strategy]: {"method":"<graph rule used>"}
STEP 4 [identification_proof]: [{"rule":"...","to":"..."}]
STEP 5 [compute]: {"result": <the graph-derived answer>}
STEP 6 [answer]: {"answer": <final answer in the question's format>}
After </think>, output the final answer on its own line as 'ANSWER: ...'.

Graph nodes: J, K, C, Y, B, R, L, W
Directed edges: K->C, R->L, B->L, B->W, J->K, B->R, C->Y, K->B
Question: Find one valid nodeset that can d-separate node Y and node B.
--------------------------  one of the eight sampled traces  --------------------------
STEP 1 [graph_extract]: {"nodes":["J","K","C","Y","B","R","L","W"],
    "edges":[["K","C"],["R","L"],["B","L"],["B","W"],
             ["J","K"],["B","R"],["C","Y"],["K","B"]]}
STEP 2 [query_id]: {"task":"d_separation_nodeset","targets":["Y","B"]}
STEP 3 [strategy]: {"method":"node_insertion"}
STEP 4 [identification_proof]: [{"rule":"node_insertion","to":"C"}]
STEP 5 [compute]: {"result": ["C"]}
STEP 6 [answer]: {"answer": ["C"]}
ANSWER: C
\end{verbatim}
\end{footnotesize}
\caption{A CLEAR clean-core problem exactly as the policy receives it, above the rule, and
one of the eight traces it returned, below. Both are quoted from the evaluated pool, with one
long JSON line wrapped to fit. The STEP lines in the prompt are the schema the model is
asked to fill, and the lines below the rule are what it produced. This trace attains the highest
score in its pool. Three distinct node sets in that pool satisfy the predicate, so more than
one answer here is accepted.}
\label{fig:supp-clearitem}
\end{figure*}
\begin{figure}[t]
\centering
\small
\setlength{\tabcolsep}{4.5pt}
{\footnotesize\raggedright
\emph{Edges:} $U\!\to\!M$, $O\!\to\!G$, $Q\!\to\!G$, $Q\!\to\!U$, $U\!\to\!G$,
$Q\!\to\!W$, $U\!\to\!W$, $W\!\to\!V$, $M\!\to\!D$\par
\emph{Question:} find one set that $d$-separates $Q$ and $O$.\par}
\vspace{4pt}
\begin{tabular}{@{}clcc@{}}
\toprule
$i$ & Answer & Score & Valid \\
\midrule
1 & $\{U,W\}$ & 1 & yes \\
2 & $\{V,D\}$ & 5 & yes \\
3 & $\{G\}$ & 4 & no \\
4 & $\{M\}$ & 5 & yes \\
5 & \texttt{\{nodeset": ["G"]\}} & 0 & no \\
6 & $\{G\}$ & 0 & no \\
\bottomrule
\end{tabular}
\caption{One CLEAR pool, showing the traces the checker scored. Three candidates give three
different separating sets and one invalid answer is returned twice, so exact plurality
returns $\{G\}$ and is wrong. $G$ is a collider on $Q\!\to\!G\!\leftarrow\!O$, so conditioning
on it opens the path it was meant to block. The checker returns candidate 2, the earliest at
the highest score. Candidate 5 is reproduced exactly as the model emitted it.}
\label{fig:supp-pool}
\end{figure}

Figure~\ref{fig:supp-pool} shows what the mechanism looks like in one pool. Because $G$ is
a collider on the only path between $Q$ and $O$, that path is already blocked, and any set
avoiding $G$ and its descendants separates the two nodes. Three of the six candidates shown
find such a set, and each finds a different one. Conditioning on the collider
instead opens the path, and that single invalid answer is returned twice, which is more
often than any one of the three correct answers. Exact plurality therefore returns it.
Nothing about the checker's decision depends on which of the three valid sets a trace
happens to name.

Figure~\ref{fig:supp-ladder} does the same for the three prose levels, on one
\texttt{sachs} problem that appears at all three. At L1 the text names the edges and
extraction is close to transcription. At L2 the same structure is described as a mechanism
and no edge list appears. At L3 it is an ordinary narrative. A trace has to recover the
same confounding of PKA and Raf by PKC from whichever telling it is handed, which is what
the drop in exact graph recovery from .826 to .336 across the ladder measures.

\begin{figure}[t]
\centering
\small
\setlength{\tabcolsep}{3pt}
{\footnotesize\raggedright
\emph{Variables:} PKA, PKC, Raf.\ \ \emph{Question:} identify a set of variables sufficient
to adjust for confounding when estimating the total effect of PKA on Raf.\par}
\vspace{4pt}
\begin{tabular}{@{}p{.05\columnwidth}p{.86\columnwidth}@{}}
\toprule
L1 & PKA has a direct effect on Raf. PKC has a direct effect on PKA and Raf. \\
\addlinespace
L2 & In cellular signaling pathways, PKC activation can drive Raf activity directly and
indirectly through its effect on PKA. Specifically, PKC influences Raf partly by increasing
PKA levels, which in turn shapes Raf activity. This cascade highlights how PKC plays a
central role in modulating Raf through multiple mechanisms. \\
\addlinespace
L3 & In cellular signaling pathways, PKC activity drives both PKA and Raf activation. PKA
then shapes Raf's activity level through its downstream effects, indicating a cascade where
PKC influences Raf both directly and indirectly through PKA. \\
\bottomrule
\end{tabular}
\caption{One \texttt{sachs} problem at the three levels of indirection, quoted from the
evaluation corpora. The variables, the question and the source subgraph
$\mathrm{PKC}\!\to\!\mathrm{PKA}\!\to\!\mathrm{Raf}$ with
$\mathrm{PKC}\!\to\!\mathrm{Raf}$ are the same at every level, so only the telling changes.
The valid adjustment set is $\{\mathrm{PKC}\}$ throughout.}
\label{fig:supp-ladder}
\end{figure}

\section{Results on the CLEAR Clean Core}
\label{sec:clear-supp}

\subsection{Every policy, against a set-aware baseline}
Table~\ref{tab:supp-policy} gives the complete policy comparison on the same
126 problems. The exact-plurality and Jaccard-medoid columns distinguish exact-string
fragmentation from set-aware agreement. The medoid improves over exact voting for six of
eight policies, and CALVER has the highest point estimate in every row. The unadapted
7B row is the smallest effect and its interval includes zero. Every interval for the seven
adapted or larger policies excludes zero.

The NF4 rows isolate quantization at 7B and 32B. Quantization reduces the supply of usable
7B candidates, whereas the 32B NF4 result remains close to its full-precision counterpart.
In both cases, the advantage over set-aware aggregation remains positive.

A surface-form normalization audit restores a common 126-problem denominator across all
policies. Cycle witnesses written as arrow chains or ordered lists are canonicalized to the
same walk before the unchanged cycle predicate is applied, so the set of problems entering
a comparison stays fixed whatever answer style a policy prefers. The canonicalization is
additive, so rows for policies
that already emitted the arrow form are unchanged, and every row of
Table~\ref{tab:supp-policy} rests on the same 126 problems.

\begin{table*}[t]
\centering
\small
\setlength{\tabcolsep}{5.2pt}
\begin{tabular}{lrrrrrr}
\toprule
Policy & Units & First & Exact plurality & Set medoid & CALVER & $\Delta$ vs. medoid (95\% CI) \\
\midrule
Qwen 7B base & 360 & 23.6 & 31.4 & 32.2 & 36.7 & $+4.4$ [$-0.8$, 9.7] \\
Qwen 7B SFT & 377 & 23.6 & 28.9 & 31.6 & 43.8 & $+12.2$ [6.9, 17.8] \\
Qwen 7B verifier-GRPO & 374 & 27.5 & 33.2 & 35.3 & 46.0 & $+10.7$ [5.6, 15.8] \\
Qwen 7B NF4 & 372 & 14.0 & 17.7 & 20.4 & 38.4 & $+18.0$ [13.3, 23.0] \\
Mistral NeMo 12B & 378 & 21.7 & 27.5 & 28.3 & 47.4 & $+19.0$ [14.6, 24.1] \\
Qwen 14B & 378 & 41.8 & 45.8 & 44.4 & 68.5 & $+24.1$ [18.8, 29.6] \\
Qwen 32B & 378 & 47.1 & 48.7 & 47.4 & 62.4 & $+15.1$ [10.1, 20.6] \\
Qwen 32B NF4 & 378 & 44.4 & 51.6 & 52.1 & 61.9 & $+9.8$ [5.6, 14.6] \\
\bottomrule
\end{tabular}
\caption{CLEAR clean-core selection accuracy (\%) at $K=8$. Every policy covers the same
126 source problems over three seeds. Units counts the problem--seed observations retained
under the parse-admissibility rule, which is applied identically to every selector.
Intervals are clustered by problem.}
\label{tab:supp-policy}
\end{table*}

\subsection{Every selector on one frozen pool}
The common-row comparison in Table~\ref{tab:supp-comparator} removes denominator differences
between scorer logs. The learned reward model is the strongest generic comparator, at
30.5\%. CALVER reaches 42.1\%. Every interval is paired and clustered by problem.

\begin{table*}[t]
\centering
\small
\setlength{\tabcolsep}{3.4pt}
\begin{tabular}{lrr}
\toprule
Selector & Accuracy & Gain of CALVER (95\% CI) \\
\midrule
Structure-only control & 23.5 & $+18.6$ [15.6, 21.7] \\
First & 24.8 & $+17.4$ [13.9, 20.9] \\
Language-model judge & 27.4 & $+14.8$ [11.4, 18.2] \\
Model confidence & 30.1 & $+12.1$ [8.6, 15.6] \\
Skywork Reward V2 8B & 30.5 & $+11.6$ [8.2, 15.0] \\
Exact plurality & 30.9 & $+11.3$ [7.8, 14.9] \\
CALVER & 42.1 & -- \\
\midrule
Candidate coverage & 51.8 & $-9.7$ [$-12.2$, $-7.4$] \\
\bottomrule
\end{tabular}
\caption{Same-pool selector comparison on 1,111 problem--seed records from 126 problem
clusters, at $K=8$. The structure-only control keeps the graph, query, derivation, and
format components and removes the validity check. Candidate coverage is the fraction of
pools containing at least one correct candidate. It is computed only after generation, and
its final-column entry is the remaining gap to that ceiling, not a comparison with a
competing method.}
\label{tab:supp-comparator}
\end{table*}

Every method ranks the same ordered candidate texts. The judge receives the graph, query,
and trace. The reward model and the confidence score rate candidates on their own. No
selector receives a target answer.

Each selector returns exactly one index from the stripped frozen pool, and that index is
what gets evaluated. First is candidate index one, the scalar scorers return the earliest
candidate attaining their maximum, and the answer-level methods use their predeclared tie
rules over canonicalized answer strings. No index is revised after a selector has returned
it, and evaluation labels are attached only once every index is fixed. The comparison keeps
the records for which every candidate-level scorer is available, which is what removes the
denominator differences between scorer logs. Each candidate's comparator scores are stored
inside the frozen pool, so the whole table regenerates on CPU without a reward-model or
judge forward pass.

The structure-only control keeps parsing, binding, derivation, and answer-format checks and
drops only semantic validity. It lands below both the first sample and exact plurality,
which is the clearest single indication that the gain comes from the causal predicate. The
distance from the checker up to candidate coverage is the headroom the frozen pool still
leaves. Coverage is computed after selection and is a ceiling, not a deployable method.

\paragraph{Traces that do not parse.}
A trace whose typed slots do not parse can either be left out of the comparison or scored
zero and left in, where it carries its raw answer into the vote. Table~\ref{tab:supp-admission}
reports both. Every selector loses ground under the second handling, since each then has to
rank traces the checker cannot score. The margin over plurality moves the other way and
widens from 11.1 to 13.0 points. The reported row is the one with the higher absolute
accuracies and the smaller margin, so the comparison does not turn on the choice.

\begin{table}[t]
\centering
\small
\setlength{\tabcolsep}{4pt}
\begin{tabular}{@{}lrrrr@{}}
\toprule
Unparsable trace & First & Plurality & CALVER & $\Delta$ \\
\midrule
Left out (reported) & 24.8 & 31.1 & 42.2 & $+11.1$ \\
Scored zero, left in & 20.3 & 25.8 & 38.8 & $+13.0$ \\
\bottomrule
\end{tabular}
\caption{Handling of traces whose typed slots do not parse, on the 1,111 CLEAR records.
Candidate coverage is 51.9 under both, since the choice does not change which traces are
valid. Both rows are counted by the sensitivity scorer, whose answer normalization differs
slightly from the one used in Table~\ref{tab:supp-comparator}, so its plurality sits 0.2
above that table and its CALVER and coverage figures sit 0.1 above.}
\label{tab:supp-admission}
\end{table}

\subsection{Prompt-conditioning the reward model}
\label{subsec:supp-rmprompt}
The main reward-model row scores each candidate trace individually. A natural objection is
that the model would do better with the problem in front of it. We therefore scored every
candidate a second time with the full problem statement supplied as context, and reran the
selection on the identical frozen pool. Problem conditioning is worth 5.5 points to the
reward model. Even so, a fully informed 8B preference model only matches exact plurality on
these candidates, 25.7 against 25.8, and the symbolic score reaches 38.7. The three figures
come from one end-to-end re-grading of the pool and are read against each other, not
against Table~\ref{tab:supp-comparator}. What separates the reward model from the checker
is the ability to check the graph predicate, and problem context alone closes little of it.

\subsection{Distilled process-reward control}
\label{subsec:supp-prm}
A second learned baseline predicts the symbolic score directly, in the spirit of training an
aggregator over sampled solutions \citep{zhao2025aggregation}. The model is a regression
head on frozen Qwen2.5-0.5B-Instruct with a LoRA adapter, trained with mean-squared error
on the verifier's $0$--$6$ output. It is evaluated as a best-of-$K$ selector on the same
nine frozen ATE pools: three policies, three seeds, 500 problems, and $K=8$. This is a
distillation control with verifier-generated labels, not human-preference or answer labels.

\begin{table}[t]
\centering
\small
\setlength{\tabcolsep}{4pt}
\begin{tabular}{lrr}
\toprule
Selector ($n=4{,}500$, $K=8$) & Accuracy & vs.\ vote (95\% CI) \\
\midrule
First sample & 76.6 & $-0.4$ \\
Exact plurality & 77.0 & -- \\
Learned PRM (distilled) & 77.1 & $+0.1$ [$-1.0$, $+1.3$] \\
CALVER (symbolic) & 79.9 & $+2.9$ [$+2.0$, $+3.7$] \\
\midrule
Candidate coverage & 84.3 & $+7.3$ \\
\bottomrule
\end{tabular}
\caption{A process reward model distilled from the symbolic score does not inherit its
selection benefit. Intervals are paired bootstraps clustered by problem. The PRM interval
includes zero, and CALVER exceeds the PRM by $2.8$ points [$+1.6$, $+3.9$].}
\label{tab:supp-prm}
\end{table}

At this budget the distilled PRM is statistically indistinguishable from exact plurality,
and the symbolic score it was trained to imitate stands 2.8 points higher [$+1.6$, $+3.9$] on
identical candidates. Its best validation RMSE is 0.83 on the $0$--$6$ scale, small in
absolute score units and still enough to reverse the ordering between adjacent top scores,
which is the only quantity best-of-$K$ selection consumes. That property also settles how the
two selectors behave as the budget grows. Equation~\eqref{eq:supp-increment} makes
improvement with $K$ depend on precision being nondecreasing in the score, so an imitator
that inverts adjacent top scores gains nothing from a larger pool, while the score it copies
keeps climbing toward the coverage ceiling.

\subsection{Judge scaling}
\label{subsec:supp-judgescale}
The direct objection to a fixed symbolic checker is that a large enough language model,
given the same graph, should reach the same decisions. We test that at two scales and two
families on one frozen Qwen-7B candidate pool of 498 CLEAR items. The judges are
Mistral NeMo-Instruct-12B and Qwen2.5-72B-Instruct
\citep{mistral2024nemo,qwen2024qwen25}. Each receives exactly the information the checker
receives, and selection uses the judge's own validity score. This pool is broader and
harder than the 126-problem clean core, with a candidate-coverage ceiling of .231 rather
than .518, so the quantity to read across tables is the distance from voting.

\begin{table*}[t]
\centering
\small
\setlength{\tabcolsep}{3.6pt}
\begin{tabular}{lrrr}
\toprule
Selector & Acc. & $\Delta$ vs.\ vote & 95\% CI \\
\midrule
First sample & .072 & $-1.0$ & --- \\
Nemo-12B judge (different family) & .060 & $-2.2$ & [$-4.6$, $+0.2$] \\
Majority & .082 & --- & --- \\
Qwen2.5-72B judge (same family) & .086 & $+0.4$ & [$-2.0$, $+3.0$] \\
\midrule
Candidate coverage & .231 & --- & --- \\
\bottomrule
\end{tabular}
\caption{Language-model judges as selectors on one frozen 498-item CLEAR pool. Both
intervals include zero, so neither judge is distinguishable from answer voting. Measured
from the first sample, the 72B judge closes 8.9\% of the distance to candidate coverage and
the 12B judge closes none of it.}
\label{tab:supp-judgescale}
\end{table*}

The two judges fail in different ways, and that difference is the informative part. The
12B judge assigns its maximum validity score to 78\% of candidates, so it
barely discriminates among them. The 72B judge produces a far broader score distribution
and, on inspection, writes substantive structural critiques that name a missed backdoor
path or confirm that none remains open. Its selected-answer accuracy is nevertheless
indistinguishable from exact plurality. A six-fold increase in judge size moves the
estimate by 2.6 points and costs one neural forward pass per candidate, against a
deterministic check measured in milliseconds on CPU.

\subsection{Scaling with the number of samples}
Table~\ref{tab:supp-kcurve} reports every frozen-prefix point for the base, SFT, and
verifier-GRPO policies. The two selectors coincide at $K=1$ by construction and separate
from $K=2$ onward. On the SFT and GRPO policies voting flattens after $K=16$ while
verification keeps climbing, which is where the widening gap comes from. The base policy
starts lower and its gap at $K=32$ is smaller, at 10.3 points, and it stays open throughout.
Every point is a prefix of one frozen $K=32$ pool, generated separately from the $K=8$ pools
of Table~\ref{tab:supp-policy}. The two tables therefore describe different samples of the
same 126 problems. Taking prefixes of a single pool is how the budget is varied in
deployment, so the curve measures the setting the method is actually used in.

\begin{table*}[t]
\centering
\small
\setlength{\tabcolsep}{4.5pt}
\begin{tabular}{c|rr|rr|rr}
\toprule
& \multicolumn{2}{c|}{Qwen 7B base} & \multicolumn{2}{c|}{Qwen 7B SFT} & \multicolumn{2}{c}{Qwen 7B verifier-GRPO} \\
$K$ & Vote & CALVER & Vote & CALVER & Vote & CALVER \\
\midrule
1 & 13.2 & 13.2 & 20.6 & 20.6 & 23.0 & 23.0 \\
2 & 15.9 & 21.4 & 22.2 & 29.9 & 24.6 & 33.3 \\
4 & 21.4 & 31.0 & 27.5 & 38.1 & 30.4 & 39.9 \\
8 & 28.3 & 37.8 & 30.4 & 44.7 & 35.4 & 48.1 \\
16 & 34.7 & 42.6 & 32.5 & 51.6 & 36.5 & 55.0 \\
32 & 37.6 & 47.9 & 32.5 & 57.9 & 36.8 & 60.6 \\
\bottomrule
\end{tabular}
\caption{CLEAR clean-core accuracy (\%) as $K$ increases. Every row uses a prefix of the
same frozen $K=32$ pool.}
\label{tab:supp-kcurve}
\end{table*}

\subsection{Graph class, query, and alternative valid answers}
Table~\ref{tab:supp-task} shows that gains are largest on $d$-separation and backdoor
selection, where several valid sets frequently exist. Directed-cycle detection has a
compact witness space and changes little. In an audited subset, 11 of 21 selected objects
that pass the graph-validity check differ from CLEAR's listed example. Manual and
programmatic
checks confirm that the alternatives satisfy the graph predicate.

\begin{table}[t]
\centering
\small
\setlength{\tabcolsep}{3.2pt}
\begin{tabular}{llrrr}
\toprule
Graph & Query & Units & Vote & CALVER \\
\midrule
DAG & $d$-separation & 90 & 44.4 & 75.6 \\
DAG & backdoor set & 36 & 25.0 & 38.9 \\
DAG & blocked path & 53 & 1.9 & 17.0 \\
ADMG & backdoor set & 36 & 8.3 & 36.1 \\
ADMG & blocked path & 54 & 5.6 & 11.1 \\
Directed & cycle witness & 108 & 49.1 & 50.9 \\
\bottomrule
\end{tabular}
\caption{Qwen 7B SFT clean-core accuracy (\%) by graph class and task at $K=8$.}
\label{tab:supp-task}
\end{table}

\subsection{Where the gain comes from, pool by pool}
\label{subsec:supp-frag-empirical}
Theorem~\ref{thm:supp-frag} conditions on an invalid answer carrying strictly more mass
than any valid one. That is checkable per pool, so we can locate the gain instead of only
aggregating it. Of the clean-core pools holding at least one valid candidate, the condition
holds on 39\% of them. A pool with no valid candidate leaves every selector equal, and the
pools that do hold one make up exactly the 51.8\% candidate coverage reported above.

\begin{table*}[t]
\centering
\small
\setlength{\tabcolsep}{4pt}
\begin{tabular}{lrrrr}
\toprule
Stratum & $n$ & Vote & CALVER & $\Delta$ \\
\midrule
Invalid answer is the strict plurality & 226 & \phantom{0}0.0 & 68.6 & $+68.6$ \\
Otherwise & 350 & 82.0 & 78.6 & $-3.4$ \\
\midrule
All pools with headroom & 576 & 49.8 & 74.7 & $+24.8$ \\
\bottomrule
\end{tabular}
\caption{Per-pool decomposition on the clean core. CALVER rescues 197 pools and loses
54. An invalid plurality is what defines the first stratum, so its vote entry is $0.0$ by
construction. The measured quantities there are the stratum's prevalence and CALVER's
accuracy within it.}
\label{tab:supp-frag-empirical}
\end{table*}

The decomposition is computed end to end by the study's clean-core scorer, which
recomputes graph validity for every sampled trace and ranks candidates by the six components
it recovers. Voting and the validity-ranked selector are computed on that single
basis, so the strata, the counts, and the differences within the table are exact.

The aggregate 24.8-point gain is the net of two cases. There are 197 repairs where the plurality is
invalid, against 54 changed selections where plurality already concentrates on a valid
answer, a ratio of 3.6 to 1. This locates the fragmentation mechanism instead of only
aggregating over it.

\subsection{Tie robustness}
Maximum scores are frequently tied because the score has seven values, so the tie rule is
worth auditing directly. Table~\ref{tab:supp-ties} recomputes accuracy for all eight
policies under three deterministic tie rules. On seven of the eight, resolving every tie to
the worst available candidate costs at most 0.6 points against the deployed first-index
rule, and six of those are within 0.3. The widest band is 4.3 points, and even under
adversarial resolution that row stands far above voting (.341 against .177). The gain is
a property of the score, not of the order in which candidates happen to arrive.

\begin{table*}[t]
\centering
\small
\setlength{\tabcolsep}{5pt}
\begin{tabular}{lrrrrr}
\toprule
Policy & Tie rate & Worst tied & First index & Best tied & Vote \\
\midrule
Qwen 14B & .910 & .685 & .685 & .696 & .458 \\
Qwen 32B & .675 & .622 & .624 & .624 & .487 \\
Qwen 32B NF4 & .638 & .616 & .619 & .622 & .516 \\
Qwen 7B SFT & .767 & .438 & .438 & .443 & .289 \\
Qwen 7B verifier-GRPO & .727 & .457 & .460 & .463 & .332 \\
Mistral NeMo 12B & .796 & .471 & .474 & .481 & .275 \\
Qwen 7B NF4 & .753 & .341 & .384 & .401 & .177 \\
Qwen 7B base & .506 & .361 & .367 & .381 & .314 \\
\bottomrule
\end{tabular}
\caption{Tie-rule ablation on the clean core, for the same eight policies and the same
problem--seed units as Table~\ref{tab:supp-policy}. ``Worst tied'' resolves every top-score
tie to an incorrect candidate whenever one exists.}
\label{tab:supp-ties}
\end{table*}

\section{External Benchmarks and Transfer}
\label{sec:external-supp}

\subsection{Published Bayesian networks with supplied graphs}
We form nontrivial treatment--outcome pairs from ten standard bnlearn networks. Each prompt
contains the query-relevant ancestral subgraph. The answer is any set satisfying the strict
backdoor criterion, and no canonical set is supplied to the selector. Table~\ref{tab:supp-bnlearn}
shows every network.

\begin{table*}[t]
\centering
\small
\setlength{\tabcolsep}{4.7pt}
\begin{tabular}{lrrrrr}
\toprule
Network & Problems & Vote & CALVER & Gain & 95\% CI \\
\midrule
win95pts & 27 & 25.9 & 55.6 & $+29.6$ & [14.8, 48.2] \\
mildew & 41 & 31.7 & 61.0 & $+29.3$ & [14.6, 43.9] \\
alarm & 17 & 64.7 & 88.2 & $+23.5$ & [5.9, 47.1] \\
barley & 43 & 37.2 & 60.5 & $+23.3$ & [11.6, 37.2] \\
hepar2 & 98 & 30.6 & 48.0 & $+17.3$ & [10.2, 25.5] \\
insurance & 115 & 52.2 & 66.1 & $+13.9$ & [7.8, 20.9] \\
water & 31 & 32.3 & 41.9 & $+9.7$ & [0.0, 22.6] \\
child & 20 & 60.0 & 60.0 & $0.0$ & -- \\
hailfinder & 5 & 60.0 & 60.0 & $0.0$ & -- \\
sachs & 13 & 7.7 & 7.7 & $0.0$ & -- \\
\midrule
Pooled & 410 & 39.8 & 56.8 & $+17.1$ & [13.4, 20.7] \\
\bottomrule
\end{tabular}
\caption{Find-one-valid adjustment-set selection on independently sourced Bayesian
networks.}
\label{tab:supp-bnlearn}
\end{table*}

The three zero-gain networks contain no correctness diversity among the sampled
top-scoring candidates, so the selector is inert. The mechanism predicts exactly that,
since reranking can improve accuracy only when the frozen pool holds candidates that differ
in validity.

\subsection{The same networks rendered as prose at three levels}
\label{subsec:supp-bnprose}
Each trace reconstructs a graph from text and proposes an adjustment set. The checker sees
only that reconstructed graph. Selected-answer correctness is measured on the retained
source graph. The same 132 treatment--outcome problems are verbalized three times, so the
levels differ only in how directly the structure is stated:

\begin{itemize}
\item \textbf{L1, edge-stated:} the scenario names each causal edge in words.
\item \textbf{L2, mechanism-explicit:} the scenario describes structural mechanisms with
no edge list.
\item \textbf{L3, naturalistic:} ordinary narrative with distractors and reordered clauses.
\end{itemize}

L2 and L3 prose is written by Qwen2.5-14B-Instruct and admitted only if the
predeclared round-trip extractor reproduces the source subgraph exactly (edge-F1 $=1$).
The filter runs before candidate generation and uses neither evaluated model outputs nor
selector scores, and the problems it retains are the Probs column of
Table~\ref{tab:supp-bnprose}. The principal evaluation policy is unadapted
Qwen2.5-7B-Instruct at $K=8$ over three seeds, with one seed at L1.

\begin{table*}[t]
\centering
\small
\setlength{\tabcolsep}{4.4pt}
\begin{tabular}{llrrrrrrrrrr}
\toprule
Level & Nets & Probs & Units & Parse & Edge-F1 & Exact $G$ & Vote & Struct. & CALVER & Coverage & $\Delta$ vs.\ vote (95\% CI) \\
\midrule
\multicolumn{12}{l}{\emph{Qwen2.5-7B-Instruct, unadapted}} \\
L1 edge-stated & 9 & 132 & 132 & .97 & .950 & .826 & 50.8 & 48.5 & 75.0 & 79.5 & $+24.2$ [17.4, 31.8] \\
L2 mechanism & 8 & \phantom{0}82 & 246 & .95 & .821 & .365 & 31.3 & 29.3 & 45.5 & 52.0 & $+14.2$ [\phantom{0}8.5, 20.7] \\
L3 naturalistic & 8 & \phantom{0}91 & 273 & .94 & .810 & .336 & 33.3 & 30.8 & 50.9 & 54.9 & $+17.6$ [11.4, 24.2] \\
\midrule
\multicolumn{12}{l}{\emph{Mistral NeMo-12B-Instruct, identical prompts and seeds}} \\
L2 mechanism & 8 & \phantom{0}82 & 246 & 1.00 & .948 & .673 & 32.5 & 31.3 & 52.8 & 55.3 & $+20.3$ [13.4, 27.2] \\
L3 naturalistic & 8 & \phantom{0}91 & 273 & .98 & .936 & .642 & 36.3 & 30.8 & 54.2 & 58.2 & $+18.0$ [11.4, 24.9] \\
\bottomrule
\end{tabular}
\caption{Published Bayesian networks rendered as prose at three levels of indirection, for
two policies from different model families. Units are problem--seed observations, and
intervals are paired bootstrap clustered by problem. ``Struct.'' is the structure-only
control, which
keeps every trace-format component and removes only the graph-validity check.}
\label{tab:supp-bnprose}
\end{table*}

Exact graph recovery more than halves between L1 and the two prose levels, and the gain
over voting does not follow it down. That is what query-local transfer predicts, since a
candidate can preserve the relations the requested adjustment decision depends on while
misreporting distant edges. The structure-only control stays below voting at every level,
which points to semantic validity and not schema compliance as the source of the
improvement. The L2 and L3 intervals overlap, so their point-estimate difference is not
interpreted.

\paragraph{A second model family.}
The lower half of the table repeats L2 and L3 with Mistral NeMo-12B-Instruct on identical
prompts and seeds. Nemo reconstructs these graphs roughly twice as often as Qwen 7B and its
gains are comparable, so the advantage survives both a change of model family and a wide
range of reconstruction fidelity.

\paragraph{Round-trip-filter sensitivity.}
We rebuild both prose levels at a relaxed edge-F1 threshold of .8 and rerun the same
evaluation. The relaxed build admits almost every problem and restores all nine networks.
Table~\ref{tab:supp-bnprose-sens} compares the two filters.

\begin{table*}[t]
\centering
\small
\setlength{\tabcolsep}{3.4pt}
\begin{tabular}{llrrrrr}
\toprule
Level & Filter & Probs & Exact $G$ & Vote & CALVER & $\Delta$ (95\% CI) \\
\midrule
L2 & $=1.0$ & \phantom{0}82 & .365 & 31.3 & 45.5 & $+14.2$ [\phantom{0}8.5, 20.7] \\
L2 & $\ge.8$ & 131 & .206 & 26.5 & 39.2 & $+12.7$ [\phantom{0}8.1, 17.6] \\
L3 & $=1.0$ & \phantom{0}91 & .336 & 33.3 & 50.9 & $+17.6$ [11.4, 24.2] \\
L3 & $\ge.8$ & 129 & .215 & 26.1 & 40.8 & $+14.7$ [10.3, 19.4] \\
\bottomrule
\end{tabular}
\caption{Filter-strictness sensitivity. Relaxing the round-trip threshold admits 99\% and
98\% of problems and restores all nine networks. Absolute accuracies fall, because less
faithful prose is harder to read, but the advantage over voting persists with intervals
excluding zero.}
\label{tab:supp-bnprose-sens}
\end{table*}

Less faithful prose is harder to read, so graph recovery and absolute accuracy both fall.
The gain over voting attenuates by under three points at either level and all four
intervals still exclude zero. The structure-only control also stays below voting on the
relaxed L2 corpus, by 4.3 points [$-8.4$, $-0.5$]. Nothing about the selection effect is
created by the strict threshold.

\paragraph{Network attrition under the strict filter.}
The strict build contains nine networks at L1 and eight at L2/L3. Hailfinder has too few
L1 items, and no mildew narrative reaches edge-F1 $=1$ at L2 or L3. Mildew returns under
the relaxed threshold. The per-network retained counts are stored with the evaluation corpus
and are used unchanged by every selector.

\subsection{Independent backdoor-adjustment benchmark}
\label{subsec:supp-wang}
The benchmark of \citet{wang2024generalization} supplies independently generated DAGs and
treatment--outcome pairs. We use the confounded six-node adjustment-set split, normalize
variable names, and grade proposed sets with the strict backdoor predicate. This avoids
depending on a released grader that does not enforce the no-descendants condition. Both
voting and CALVER operate on the same proposed adjustment-set objects.

\begin{table*}[t]
\centering
\small
\setlength{\tabcolsep}{4pt}
\begin{tabular}{lrrrr}
\toprule
Policy & $n$ & Set vote & CALVER & Gain (95\% CI) \\
\midrule
Qwen 7B SFT & 167 & 53.9 & 80.2 & $+26.3$ [19.8, 32.9] \\
Qwen 1.5B SFT & 167 & 1.8 & 32.9 & $+31.1$ [24.0, 38.3] \\
\bottomrule
\end{tabular}
\caption{Backdoor-adjustment selection on the independent benchmark at $K=8$.}
\label{tab:supp-cibench}
\end{table*}

The weaker 1.5B policy produces a highly fragmented pool, so no single set commands a
plurality. A validity test is unaffected by how the correct mass is divided among sets.
The result corroborates the selection mechanism on independently sourced graphs.
A direct graph algorithm can solve the formal task once the DAG is available, so what the
experiment tests is candidate selection and not a need for language modeling.

\subsection{CausalGraph2LLM}
The adapter covers mediator and post-intervention reachability across prose, JSON,
adjacency, multi-node, and Graphviz encodings \citep{sheth2025cg2llm}. The
formal adapter reproduces all 4,800 released labels. Model evaluation uses 1,200 distinct
items per run while preserving both tasks, all five encodings, and all twelve graphs.
Pooling two policies and two seeds yields 4,800 problem--seed units over those 1,200
problems. Table~\ref{tab:supp-cg2llm} gives the $K=16$ decomposition.

\begin{table*}[t]
\centering
\small
\setlength{\tabcolsep}{3.1pt}
\begin{tabular}{lrrrr}
\toprule
Slice & Units & Vote & CALVER & Gain \\
\midrule
Mediator & 2,400 & 60.8 & 65.5 & $+4.8$ \\
Intervention & 2,400 & 50.4 & 51.2 & $+0.8$ \\
\midrule
Adjacency & 960 & 57.2 & 59.9 & $+2.7$ \\
Graphviz & 960 & 54.2 & 54.7 & $+0.5$ \\
JSON & 960 & 55.6 & 59.1 & $+3.4$ \\
Multi-node & 960 & 52.8 & 57.1 & $+4.3$ \\
Prose & 960 & 58.1 & 60.9 & $+2.8$ \\
\midrule
Pooled & 4,800 & 55.6 & 58.3 & $+2.8$ \\
\bottomrule
\end{tabular}
\caption{CausalGraph2LLM accuracy (\%), pooled across two policies and two seeds. The task
rows and the encoding rows are two complete decompositions of the same 4,800 units, and the
pooled row is the value reported in the main paper. Gains are computed before rounding.}
\label{tab:supp-cg2llm}
\end{table*}

\subsection{Graph construction from language}
\label{subsec:supp-construction}
The construction suite contains 80 causal graphs formed from eight motifs across ten
domains. Every graph is confounded, includes a descendant-of-treatment trap, and 30 graphs
admit at least two minimal valid sets. The policy is the unadapted Qwen 7B instruction
model, sampled at $K=8$ over three seeds. Each trace receives a variable legend and a
narrative, constructs a graph, and proposes a set. The checker validates the set on the
trace's graph, and evaluation retains the source graph separately.

L1 explicitly states edges in prose. L2 describes mechanisms without an edge list. L3 is a
naturalistic scenario with distractors and reordered clauses. Table~\ref{tab:supp-levels}
reports all 240 problem--seed units per level.

\begin{table}[t]
\centering
\small
\setlength{\tabcolsep}{3.0pt}
\begin{tabular}{lrrrrr}
\toprule
Level & Edge-F1 & Exact $G$ & Vote & CALVER & Gain \\
\midrule
L1 edge-stated & .959 & .390 & .37 & .87 & $+.50$ \\
L2 mechanism & .954 & .340 & .29 & .81 & $+.52$ \\
L3 naturalistic & .884 & .060 & .30 & .82 & $+.52$ \\
\bottomrule
\end{tabular}
\caption{Graph construction from text. Accuracy and reconstruction metrics are fractions;
$n=240$ per level.}
\label{tab:supp-levels}
\end{table}

The gap remains positive on all eight motifs (Table~\ref{tab:supp-motifs}). The largest
changes occur on motifs with two valid minimal sets or harmless covariates that split
surface answers. The smallest occurs on the two-confounder motif, where one compact set
dominates the pool and voting is already strongest.

\begin{table*}[t]
\centering
\small
\setlength{\tabcolsep}{5pt}
\begin{tabular}{lrrr@{\hspace{16pt}}lrrr}
\toprule
Motif & Vote & CALVER & Gain & Motif & Vote & CALVER & Gain \\
\midrule
X-side block, two sets & .17 & .88 & $+.71$ & Collider trap & .21 & .66 & $+.44$ \\
Confounder + predictor & .14 & .79 & $+.64$ & Single confounder & .39 & .80 & $+.41$ \\
Y-side block, two sets & .28 & .89 & $+.61$ & Combined, two sets & .46 & .86 & $+.40$ \\
Instrument distractor & .31 & .88 & $+.57$ & Two confounders & .60 & .93 & $+.33$ \\
\bottomrule
\end{tabular}
\caption{Per-motif graph-construction accuracy, pooled across L1--L3. Each motif contains
90 problem--seed units.}
\label{tab:supp-motifs}
\end{table*}

A predeclared hand-written stress set of 12 harder L3 narratives is scored the same way.
Mean edge-F1 is .875 and exact
recovery is 10\%. Vote reaches .31 and CALVER .81, a 50-point gain [31, 67]. Splitting
this fixed set by mean construction fidelity gives gains of 33 points on the lower half
(edge-F1 .829) and 67 on the higher half (.921). The within-set trend directly connects
formalization fidelity to selection quality.

\subsection{Extract-then-solve crossover}
A natural comparison constructs one consensus graph and solves the formal query directly.
For a fair executable comparison, we retain the 153 units per level with at least one
parseable candidate graph. Table~\ref{tab:supp-crossover} reports the result, and the
candidate-coverage column is what makes it interpretable. Extraction plus solving derives a
fresh adjustment set from the merged graph instead of choosing among the sampled answers,
so it sits outside the selection class altogether. At L1 and L2 it reaches or exceeds the
coverage ceiling, .928 against .922 and .967 against .824. No method restricted to
choosing one of the sampled candidates can match that, and CALVER sits exactly at the
ceiling at L2. At L3, where exact graph recovery falls to 17.2\%, the consensus
graph is no longer dependable and the two routes are statistically indistinguishable. The
result gives a practical rule. Extract once and solve when transcription is reliable, and
verify every candidate when readings diverge and different traces preserve different
query-local facts.

\begin{table*}[t]
\centering
\small
\setlength{\tabcolsep}{3.1pt}
\begin{tabular}{lrrrrr}
\toprule
Level & Exact $G$ & Solver & CALVER & Coverage & $\Delta$ (95\% CI) \\
\midrule
L1 & .610 & .928 & .837 & .922 & $-.092$ [$-.150,-.039$] \\
L2 & .674 & .967 & .824 & .824 & $-.144$ [$-.216,-.072$] \\
L3 & .172 & .837 & .856 & .882 & $+.020$ [$-.046,.092$] \\
\bottomrule
\end{tabular}
\caption{Consensus extraction plus exact solving versus candidate-wise verification on
parseable graph-construction units. Coverage is the post-generation ceiling set by the
presence of at least one correct proposed answer, and it bounds every selector.}
\label{tab:supp-crossover}
\end{table*}

\begin{figure}[t]
\centering
\includegraphics[width=\columnwidth]{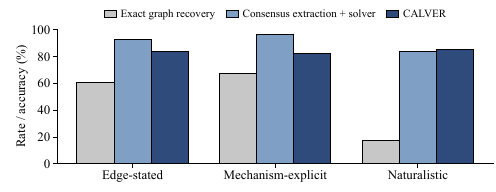}
\caption{The direct-solver advantage disappears as exact graph recovery deteriorates.}
\label{fig:supp-crossover}
\end{figure}

\paragraph{A note on routing.}
Routing between the two architectures needs a signal available at inference. The obvious
one, agreement among the candidates' own parsed edge sets, does not supply it. Median
pairwise agreement is 1.00, 1.00, and 0.92 at L1--L3, so a gate fixed in advance at $0.5$
sends every instance to the solver and the hybrid reduces to consensus extraction at all
three levels. The reason is the paper's own mechanism one level down. Traces that misread a
narrative tend to misread it the same way, so agreement stays high exactly where the
consensus graph is wrong, and correlated error defeats voting over edges as it defeats
voting over answers.

\subsection{External logic benchmark}
The mem-kk benchmark of \citet{xie2025memorization} contains dynamically generated,
unique-solution knights-and-knaves
puzzles with compound conjunction, disjunction, implication, and biconditional statements.
The unadapted Qwen 7B model restates each statement in a closed formula grammar and proposes
an assignment. A truth-table engine checks that assignment against the model's formulas.
As a fidelity check, the benchmark-to-formula evaluation code reproduces the published
solution on all 540 generated items.

\begin{table*}[t]
\centering
\small
\setlength{\tabcolsep}{3.4pt}
\begin{tabular}{crrrrr}
\toprule
People & Parse & Fidelity & Vote & CALVER & Gain (95\% CI) \\
\midrule
3 & .97 & .72 & .50 & .77 & $+.27$ [.20, .34] \\
4 & .96 & .70 & .26 & .63 & $+.37$ [.29, .44] \\
5 & .92 & .45 & .22 & .38 & $+.16$ [.09, .23] \\
\bottomrule
\end{tabular}
\caption{mem-kk, 180 puzzle--level units per size, $K=8$, three seeds, no task-specific
fine-tuning. Fidelity is exact agreement of the restated formula set with the benchmark
formula set.}
\label{tab:supp-memkk}
\end{table*}

The result separates parsing from semantic fidelity. The model emits a parseable formula
block on more than 92\% of items at every size, yet exact formalization falls to 45\% for
five people. The corresponding selector gain falls from 27--37 points to 16. Thus a
well-formed intermediate representation earns its place only when it preserves the formal
content the checker consumes.

\subsection{Independent DoVerifier selection}
To test the selection principle with an independent formal engine, we evaluate
candidate-target DoVerifier \citep{he2026doverifier} on a held-out causal-identification
corpus. Each item is gated at construction, before any candidate exists, by requiring the
prover to derive the intended estimand and to fail on the paired distractor. An item that
misses either condition is dropped and counted. All 240 generated items passed, so the
evaluation set is the full corpus and not a surviving subset. DoVerifier then receives
each candidate's own proposed target expression and attempts a proof, without the reference
target. The selector returns the earliest provable candidate, and an item whose pool
contains no provable candidate is scored incorrect.

\begin{table}[t]
\centering
\small
\setlength{\tabcolsep}{3.8pt}
\begin{tabular}{lrrr}
\toprule
Slice & Items & First & DoVerifier selection \\
\midrule
All & 240 & .496 & .800 \\
Mediator trap & 79 & .316 & .646 \\
Unconfounded & 84 & .595 & .821 \\
Confounded & 77 & .571 & .935 \\
\bottomrule
\end{tabular}
\caption{Candidate-target DoVerifier as a target-label-free selector.}
\label{tab:supp-doverifier}
\end{table}

Of the 1,920 candidates, 91.9\% emit a target expression at all and 1,013 admit a proof of
their own target. Selection accuracy equals the rate at which a pool contains at least one
provable candidate on every slice, which is what a sound prover and a fail-closed selector
together predict. The largest gains fall on the mediator-trap and confounded slices, where
a plausible observational expression is not identified. This experiment does not compare
two implementations of one score. It shows that candidate-wise formal verification improves
selection when another proof system supplies the predicate.

\section{Checking the Checker}
\label{sec:audits-supp}

\subsection{Does a high score mean a correct trace?}
Selection accuracy leaves open whether a high score is itself reliable, so we also
measure precision, recall, false-positive rate (FPR), coverage, and AUROC at the score
threshold each task uses. FPR is the fraction of incorrect candidates that reach that
threshold, and coverage is the fraction of all candidates that reach it.

\begin{table*}[t]
\centering
\small
\setlength{\tabcolsep}{2.8pt}
\begin{tabular}{lrrrrrr}
\toprule
Setting & $n$ & Precision & Recall & FPR & Coverage & AUROC \\
\midrule
CLEAR validity & 344 & .955 & .955 & .021 & .267 & .962 \\
ATE maximum deployed score & 2,376 & .913 & .451 & .061 & .246 & .713 \\
\bottomrule
\end{tabular}
\caption{How well a high score predicts a correct trace, at the threshold each task uses.}
\label{tab:supp-quality}
\end{table*}

The CLEAR row covers the audited many-satisfier subset in full, pooled over three seeds,
with no sampling step. Every candidate in that subset is included, and the threshold used
is the validity component. Its correctness label is the graph predicate
recomputed without reading CLEAR's example answer. The ATE row uses
the maximum deployed score before the treatment-exclusion and decision-margin guards. Its
lower recall reflects correct final decisions whose traces omit or fail a typed
intermediate component, and its pre-guard false positives are decomposed below.
Equation~\eqref{eq:supp-increment} identifies the quantity that governs continued
improvement with $K$, namely precision ordered in the score, and the measured curves in
Table~\ref{tab:supp-kcurve} rise at every budget through $K=32$ for all three policies.

\subsection{Graph corruption}
The candidate texts and their answers are held fixed and only the graph handed to the
checker is perturbed, so any change in lift is attributable to the structural information
the check has to work with. Figure~\ref{fig:supp-corruption} shows the result. The decline
is ordered, and it is gradual at the top. Deleting, adding, or reversing a tenth of the
edges leaves most of the benefit in place, and only a random or fully reversed graph costs
more than half of it. Query-local invariance explains why the top of that curve is so
flat. A corrupted edge outside the queried ancestral signature leaves the decision it
supports unchanged.

\begin{figure}[t]
\centering
\includegraphics[width=\columnwidth]{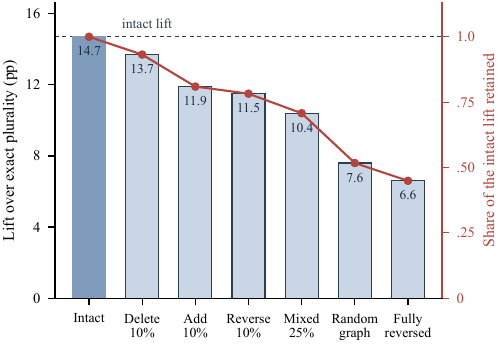}
\caption{Lift over exact plurality on the CLEAR core as the graph given to the checker is
damaged, in percentage points on the left axis and as a share of the intact lift on the
right. The candidate pool is identical in every condition.}
\label{fig:supp-corruption}
\end{figure}

\subsection{Tests of the graph routines themselves}
The graph and numerical primitives are tested independently of language-model output.
Table~\ref{tab:supp-property} gives the final scaled tests. All comparisons use independently
implemented or direct-enumeration references.

\begin{table*}[t]
\centering
\small
\setlength{\tabcolsep}{5pt}
\begin{tabular}{p{.27\textwidth}rrp{.39\textwidth}}
\toprule
Test & Cases & Disagreements & Reference and criterion \\
\midrule
DAG $d$-separation & 150,000 & 0 & NetworkX moralized ancestral implementation \citep{hagberg2008exploring} \\
Typed ADMG $m$-separation & 50,000 & 0 & independently coded active-path traversal \\
Backdoor adjustment vs. intervention & 104 & 0 & enumerated SCMs; maximum numerical error $6.66\times10^{-16}$ \\
Typed CLEAR graph round trip & 2,424 & 0 & nodes and directed/bidirected edges preserved exactly \\
Consistent node renaming & 19,176 candidates & 0 score changes & graph and trace renamed by one bijection \\
\bottomrule
\end{tabular}
\caption{Verifier property tests. These tests exercise the checker without using selection
accuracy as their reference.}
\label{tab:supp-property}
\end{table*}

\subsection{Adversarial trace modifications}
Three families probe whether the score keys on semantic structure or on surface
identifiers. The perturbations are applied to 19,176 scored candidates, of which 6,983
attain the maximum deployed score before perturbation. Consistent node renaming leaves all
19,176 scores unchanged, as an invariance control requires. The two destructive families
are read against the 6,983 maximum-score candidates, since a candidate that already failed
the component being damaged has no further score to lose. Every one of the 6,983 loses the
maximum when its final answer is swapped, and every one of the 508 whose witness path is
damaged loses it as well.

How far each score falls tracks how much semantic content the perturbation destroys.
Swapping the answer while leaving the witness intact moves every affected candidate to
exactly score three. Three components go at once: the answer stops following from the
recomputed result, the result stops surviving recomputation against the stated answer, and
the strategy check now runs on the object the trace returns. Replacing one edge of the witness path
moves every affected candidate to exactly score five, the single structural component that
edge supports. Measured over all scored candidates instead of the maximum-score ones, the
two families lower the score on 0.560 and 0.416 of cases. Most candidates had already lost
the relevant component, and a swapped binary answer can land on the right value.

\subsection{Which components carry the numeric ATE result}
A component ablation on frozen ATE pools separates structural formatting from numerical and
answer consistency. Using graph, query, strategy, and derivation components alone changes
selection over vote by 0.0 points at 1.5B and 0.2 at 7B. The complete score changes it by
1.6 and 3.8 points on those two broad pools. The main paper therefore treats ATE as a
certificate-calibration experiment. The broad ATE pools serve as numerical-score
diagnostics, and the main selected-answer results come from the executable
graph-validity tasks and the constructed-structure experiments.

\subsection{Strict ATE decision-margin audit}
\label{sec:calibration-supp}

The strict ATE experiment calibrates the numerical checks when the complete binary SCM
is part of the problem. The task asks whether the ATE exceeds the decision threshold
$\tau=0$, and the numerical tolerance is fixed at $\varepsilon=0.02$. The selector receives
neither a stored answer nor a precomputed effect. For each candidate it checks the proposed
adjustment strategy, recomputes both interventional means from the supplied CPTs, forms
their difference, and verifies the reported threshold decision.

Table~\ref{tab:supp-ate-margin} gives the complete margin decomposition for the frozen
1.5B verifier-GRPO pool. The gain rises as the target decision moves away from the
boundary. Inside the declared tolerance band, numerical agreement at $\varepsilon$ no
longer establishes that the reported and exact effects fall on the same side of the
threshold, so the strict audit marks those traces uncertified. The deployed six-component
selector still returns its ordinary maximum-score answer there, and the table reports its
accuracy so that the calibration can be read across the whole margin range.

\begin{table*}[t]
\centering
\small
\setlength{\tabcolsep}{3.4pt}
\begin{tabular}{lrrrr}
\toprule
$|\theta|$ stratum & $n$ & Vote & CALVER & $\Delta$ \\
\midrule
$[0,.02)$ & 48 & .688 & .667 & $-.021$ \\
$[.02,.05)$ & 39 & .564 & .641 & $+.077$ \\
$[.05,.15)$ & 96 & .448 & .594 & $+.146$ \\
$[.15,1]$ & 114 & .491 & .640 & $+.149$ \\
\midrule
Reported margin stratum, $|\theta|\geq .02$ & 249 & .486 & .622 & $+.137$ \\
\bottomrule
\end{tabular}
\caption{Strict ATE selection by decision margin. The reported margin-stratum interval for the
13.7-point gain is [7.2, 20.1] points, clustered by problem.}
\label{tab:supp-ate-margin}
\end{table*}

Across 2,376 candidate traces, 585 attain the maximum deployed score, and 534 of those
already carry the correct qualitative decision before any strict guard is applied. The 51
that do not divide cleanly. Forty-five sit inside the declared decision band, and the
remaining six place the treatment in its own adjustment set, so the treatment-exclusion
guard rejects them. No third failure class appears in this audit, which is what makes the
two guards in Theorem~\ref{thm:supp-sound} exhaustive on the data as well as sufficient in
the proof.

At the maximum score, and before the theorem's guards are applied, precision is .913,
recall .451, the false-positive rate .061, and coverage .246. This is the role the ATE
experiment plays in the paper. It is a direct
calibration of the numerical score and strict guards, complementary to the external many-satisfier
benchmarks where the principal quantity is selected-answer accuracy.

\subsection{Where the mechanism predicts no gain}
The fragmentation account is falsifiable in one direction. Where a query has essentially one
correct answer, no valid probability mass is available to fragment, and a validity check has
nothing to recover that voting has already lost. CLadder \citep{jin2023cladder} and
Corr2Cause \citep{jin2024corr2cause} test that prediction, since both are evaluated through
a binary interface. Across nine frozen $K=8$ pools spanning three policies and both prompt
formats, exact plurality and the symbolic selector stay within two points of each other in
every configuration. None of the differences is distinguishable under the
problem-clustered bootstrap used elsewhere. This is the predicted null, and it is reported
as one.

\section{Policy and Training Controls}
\label{sec:training-supp}

\subsection{Voting regressions without using a verifier}
The claim that plurality can hurt does not rely on CALVER scores. For each of 18 policy
configurations and five causal query types, we compare the first candidate with exact
plurality over eight candidates. Plurality is less accurate in 31 of the 90 cells.
Table~\ref{tab:supp-majority-hurts} gives the complete count by query type. The pattern is
not confined to one causal rung. It appears on observational, interventional, and
counterfactual queries.

\begin{table}[t]
\centering
\small
\setlength{\tabcolsep}{5pt}
\begin{tabular}{lrr}
\toprule
Query & Policies tested & Plurality below first \\
\midrule
ATE & 18 & 7 \\
ETT & 18 & 9 \\
Marginal & 18 & 4 \\
Conditional & 18 & 9 \\
Probability of necessity & 18 & 2 \\
\midrule
Total & 90 cells & 31 \\
\bottomrule
\end{tabular}
\caption{Exact-plurality regressions at $K=8$ relative to the first sample. This table is
independent of every verifier component. ATE is the average treatment effect; ETT is the
effect of treatment on the treated.}
\label{tab:supp-majority-hurts}
\end{table}

The fragmentation theorem gives the corresponding explanation. Repeated sampling
concentrates plurality on the most common answer atom, even when the combined probability
of several less common valid atoms is larger. A validity-aware selector instead pools those
atoms through their shared semantic predicate. The motif experiment in
Table~\ref{tab:supp-motifs} tests this mechanism directly. Gains are largest when the graph
admits multiple minimal adjustment sets or valid supersets.

\subsection{Separating the checker from the training recipe}
The clean-core experiment spans unadapted and supervised generators, verifier-guided
GRPO, NF4 quantization, larger Qwen models, and the Mistral family. The checker, score
weights, and tie rule remain fixed. The checker improves the point estimate for each
policy, so the inference method stands on its own without verifier-guided post-training.

Schema adaptation governs how much of that signal the checker can reach. In
Table~\ref{tab:supp-policy} the unadapted 7B policy carries the smallest advantage over
set-aware aggregation and the only interval that includes zero. Every adapted or larger
policy carries a substantially wider margin, with an interval excluding zero. A
policy that emits the six slots more reliably supplies more scoreable traces, and the
decision procedure is identical in every row.

Training changes the supply and format of candidates. SFT increases adherence to the
six-slot format, and policy optimization changes the distribution and diversity of
traces. Because selection operates on frozen pools, policy changes and verifier changes
can be evaluated separately.

\subsection{Quantization and model-family controls}
The 7B NF4 policy reaches .384 under the deployed tie rule and .401 under best-case
ties, compared with .177 for voting. Nemo-12B reaches .474 and .481 under the same two
rules, compared with .275 for voting. The Qwen 14B and 32B rows reach .685 and .624 under
the deployed rule. The effect appears across precision formats, parameter scales, and two
backbone families.

The two NF4 rows also measure what 4-bit quantization costs at each scale, with the
adapter, prompt, and decoding settings held fixed. At 7B the cost is real. Against the
full-precision policy the first sample falls from 23.6\% to 14.0\%, candidate coverage from
50.9\% to 40.1\%, and CALVER from 43.8\% to 38.4\%. A quantized 7B simply supplies fewer
usable traces. At 32B it is close to free, with CALVER moving
from 62.4\% to 61.9\% and coverage from 69.0\% to 68.0\%. The margin over set-aware
aggregation holds at both scales and widens at 7B, because the checker still finds the
valid candidates that survive quantization. Under the evaluated prompt and decoding
configuration, NF4 at 32B preserves nearly all of the measured selection gain.

\section{Reproducing the Results}
\label{sec:repro-supp}

\subsection{Model and sampling configuration}
The principal generator family is Qwen2.5-Instruct \citep{qwen2024qwen25}. Trainable
adapters use LoRA \citep{hu2022lora} on the $q$, $k$, $v$, and $o$ attention projections,
with rank 64 and scaling parameter 128. The core 7B SFT policy is trained for
approximately 1,000 steps on canonical six-slot
traces with AdamW, learning rate $5\times10^{-5}$, and effective batch size 16.
Verifier-guided policies use group-relative updates with a frozen reference model and the
same output format \citep{shao2024deepseekmath}. NF4 adapters follow the QLoRA
quantization setup \citep{dettmers2023qlora}. The inference method itself does not
require that training objective.

The main $K$ curves use
$K\in\{1,2,4,8,16,32\}$ and temperature .8. External runs retain the sampling parameters
stored in each run manifest. CausalGraph2LLM uses $K=16$, and the principal external
find-one-valid experiments use $K=8$. Candidate order is preserved from generation through
selection.

Training ran in bfloat16 on an academic cluster containing A100 40GB and H100 80GB GPUs.
Approximate wall-clock cost was three GPU-hours per seed at 1.5B and twelve GPU-hours per
seed at 7B. A complete 500-problem, $K=8$ evaluation takes about 45 minutes at 1.5B and two
hours at 7B per policy and seed. Symbolic scoring takes 1--8 ms per candidate on one CPU
thread, so it adds little cost beyond the samples already generated for voting.

\subsection{Frozen pools and what selectors receive}
Every selector comparison follows the same five steps:
\begin{enumerate}
\item Generate and freeze the ordered candidate pool.
\item Export a selector view containing the problem text, input-visible graph or SCM,
      candidate trace, model scores when available, and candidate index.
\item Remove benchmark answers, correctness flags, canonical adjustment sets, and any
      analysis-only field from that view.
\item Run all deployable selectors on the identical exported pool.
\item Reattach evaluation labels only after each selector has returned an index.
\end{enumerate}
An automated check rejects any comparator input still carrying
fields such as \texttt{correct}, \texttt{gold}, \texttt{answer\_key}, or a benchmark target.
The independent DoVerifier experiment receives a candidate expression as its proof target;
it never receives the benchmark's intended terminal expression.

\subsection{What each number counts, and how intervals are formed}
Accuracy is computed at the problem--seed unit unless a table explicitly reports
candidate-level classifier statistics. Paired differences are formed within each frozen
pool. Headline confidence intervals use 10,000 clustered bootstrap resamples of source
problems, and exploratory decompositions use 2,000. Paraphrases, encodings, policies, and
seeds derived from the same source problem stay in one cluster.

Set-valued answers are canonicalized only for exact voting. The Jaccard-medoid baseline
selects the candidate set with maximum mean Jaccard similarity to the other sets.
Correctness is always evaluated by the task predicate, never by equality with a listed
example.

\subsection{Code and data}
The Code and Data Package is being prepared for public release. It will contain the source
records, generated candidates with their per-candidate scores, and analysis scripts under
paths relative to its root. The typed graph parser and strict ATE verifier will sit in
\path{lib/}. The clean-core scorer, same-pool comparison, $K$-scaling report, tie-rule
ablation, set-aware baselines, and clustered paired bootstrap will sit in \path{scorers/}.
The graph-primitive property tests, extraction crossover, and routing check will sit in
\path{diagnostics/}. Each result directory will record the model and adapter identifiers,
prompt version, quantization configuration, sampling parameters, seed, corpus hash, and
scorer version for the run that produced it. The source manifest will record the external
repositories used to reconstruct CausalGraph2LLM and DoVerifier. Per-candidate comparator
scores will be stored beside the traces they score, allowing the selector comparisons in
this document to regenerate on CPU without model inference.

\subsection{Prompt formats}
For causal graph tasks, the system instruction requests exactly six named slots. The graph
slot preserves edge type. The query slot binds the named variables. The strategy slot holds
a machine-readable set, path, or intervention. The derivation record names the operation and
its arguments. The result slot holds the recomputed predicate or quantity, and the final
answer holds only the requested object or decision. The
parser accepts harmless formatting variation but rejects missing, duplicated, or conflicting
slots.

For graph-from-text tasks, STEP 1 additionally asks the model to state a directed edge list
over the variables named by the prompt. No source edge list is inserted into the verifier.
For the K\&K benchmark, the analogous slot is a closed-grammar block named
\texttt{CONSTRAINTS}. The truth-table engine verifies the candidate assignment against
that block. In both domains the model exposes a formal object, and the selector checks
whether the proposed answer follows from it.

\section{What These Results Do and Do Not Establish}
\label{sec:scope-supp}

The experiments differ in how much formal structure the problem hands over, and the
strength of the claim differs with it.

One result is proved rather than measured. When the supplied DAG and observational
distribution are correct and positivity holds, a strict ATE certificate implies the correct
threshold decision. Theorem~\ref{thm:supp-sound} establishes this, the two constructions
after it show that neither guard can be dropped, and the strict ATE audit locates every
incorrect maximum-score trace inside one guard or the other. The guarantee is stated for a
supplied graph and is not extended to graphs a candidate constructs for itself.

The rest is measured. Where the graph is supplied and the query admits several valid
answers, executable graph validity improves selection whenever the frozen pool holds
candidates that differ in validity. The CLEAR core, the \texttt{bnlearn} networks, the
independent adjustment-set benchmark, and the same-pool comparison all show this. Where
the structure has to be read out of text, candidate-local checking still improves
selected-answer accuracy even when whole-graph recovery is far from exact. The prose
ladder, the construction motifs, and the hand-written narratives measure that effect, and
the query-local analysis explains it. Where a single extraction is dependable, extracting
one graph
and solving it outright is the better architecture, and the crossover study marks where that
switch happens. Where the answer space is compact, little improvement should be expected and
little is found, which is what CLadder, Corr2Cause, and the CausalGraph2LLM intervention
slice contribute. Outside causal graphs entirely, the same candidate-wise idea carries over
once another executable predicate takes the place of graph validity, as the truth-table
check on Knights and Knaves and the independent do-calculus prover both show.

Two limits are worth stating plainly. The method selects among proposed traces and offers
no opinion on whether a real-world causal graph is scientifically correct. And in every
text-derived setting the source graph is kept outside selection and used only to measure
how well candidate-local checking transfers to the intended structure. The selected trace,
its formalization, and its component scores are all retained, so any individual decision
can be re-examined.

\section{Summary}
In short, exact voting can discard valid
reasoning when probability mass is divided across several answer strings. On frozen pools,
candidate-wise executable validity provides a stronger non-oracle selection signal than
surface frequency, model confidence, generic reward models, and reference-free judges.
The effect appears across policies, graph families, text-derived structures, an independent
do-calculus prover, and a propositional logic domain. The formal results characterize the
strict ATE guarantee, the fragmentation mechanism, finite-$K$ selection, and the
query-local graph information needed for transfer.

\section{Statement}
Language models were used while preparing this submission, to edit and polish text and code the authors had written. The results and the analysis are the authors' own. Language models also appear throughout the work as the systems under study, which the experimental sections
describe.

\FloatBarrier
\bibliography{supplement}

%% file: Chapters/0_Abstract.tex
\begin{abstract}
Self-consistency assumes the most frequent answer among sampled reasoning traces is the most reliable, but this can fail in causal reasoning: samples often repeat the same confounding error, and votes fragment across multiple valid answers, letting an invalid answer win despite a valid minority trace. We introduce CALVER (Causal Axiom-Level VERification), a training-free symbolic verifier that scores structured traces against Pearl's causal criteria, including $d$-separation, backdoor adjustment, and intervention, and selects the highest-scoring candidate without consulting a reference answer. On CLEAR find-one-valid queries that admit multiple graph-valid answers, CALVER reaches 42.1\% where plurality, a reward model, an LLM judge, and model confidence remain near 30\% on identical frozen pools. Scaling the judge to 72B does not close the gap. In an audited clean-core subset, 11 of 21 graph-valid CALVER selections differ from the benchmark's listed answer while still satisfying the requested predicate. The advantage widens with the sampling budget and reproduces across ten published Bayesian networks, a second model family, and settings where the model must build the graph from text. CALVER also improves thresholded average-treatment-effect decisions against exact ground truth, generalizes to logic under a truth-table checker, and scores each candidate in milliseconds on CPU. CALVER needs only a causal structure, supplied outright or built from the text; wherever that holds, selection can aggregate via causal validity.
\end{abstract}

%% file: Chapters/1_Introduction.tex
\section{Introduction}
Chain-of-thought (CoT) prompting elicits explicit reasoning \citep{wei2022chain} from large language models (LLMs). Building upon CoT, many test-time methods have been developed to solve complex reasoning problems by scaling up inference-time compute. Self-consistency samples several traces and returns the most frequent answer \citep{wang2022self}. Repeated sampling and sample selection are now common ways to spend that compute \citep{snell2024scaling,brown2024large}. These sampling methods work when correct traces concentrate on one answer more strongly than incorrect traces. However, for causal reasoning tasks, exact agreement becomes sparse when several outputs are viable \citep{wang2024soft} because multiple outputs can be fully correct. For example, a query asking for any set that satisfies Pearl's backdoor criterion may admit $\{Z_1\}$, $\{Z_2\}$, and $\{Z_1,Z_2\}$ as valid adjustment sets \cite{chen2024clear}. Several conditioning sets may likewise $d$-separate the same variables \citep{pearl2009causality,peters2017elements}. In causal reasoning problems, these valid answers split the vote. A familiar invalid answer can then become the single largest mode, and additional samples reinforce it.
\begin{figure*}[htb]
\centering
\includegraphics[width=0.94\textwidth]{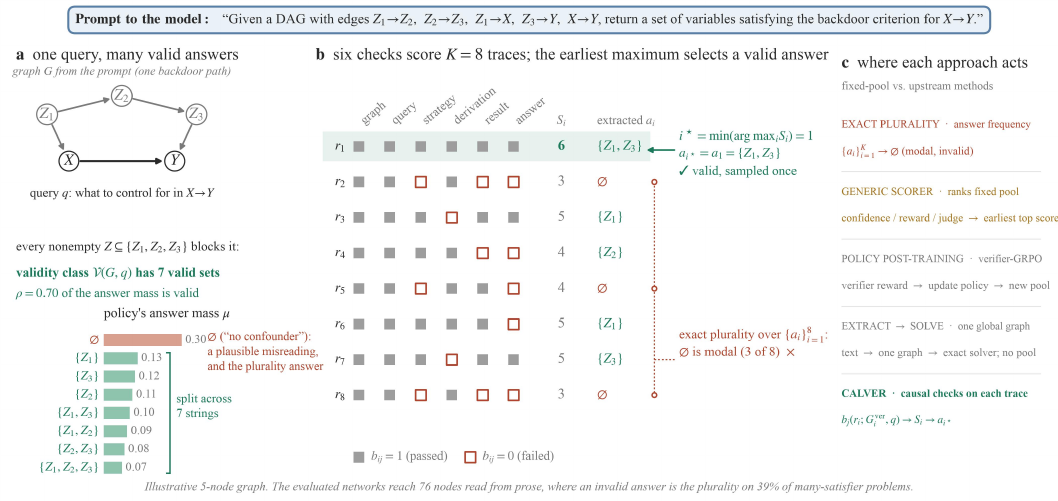}
\caption{\textbf{CALVER as a best-of-$K$ selector.} \textbf{(a)} In the illustrative graph, every nonempty subset of $\{Z_1,Z_2,Z_3\}$ satisfies the backdoor criterion for the effect of $X$ on $Y$. Valid probability mass $0.70$ is split across seven strings, while the invalid answer $\varnothing$ is the single largest mode. \textbf{(b)} Six deterministic checks score $K=8$ sampled traces; plurality returns $\varnothing$, whereas CALVER returns the earliest maximum-score valid trace. \textbf{(c)} Voting and generic scorers rank a fixed pool, verifier-guided post-training changes the policy, and extract-then-solve bypasses the pool. The evaluated networks reach 76 nodes. Theorem~\ref{thm:sound} establishes soundness for the strict ATE certificate defined in the Method.}
\label{fig:overview}
\end{figure*}

To address this, we propose \textbf{C}ausal \textbf{A}xiom-\textbf{L}evel \textbf{VER}ification (\textbf{CALVER}) to improve inference-time scaling for causal reasoning tasks. CALVER uses a key property of causal inference: the graphical criteria that create answer multiplicity also determine whether a candidate belongs to the query's \emph{validity class}. For the causal models, queries, and available observational information assumed here, validity is algorithmically decidable through standard graph computations: ancestral restriction, $d$-separation and its mixed-graph extension $m$-separation, intervention-graph surgery (deleting incoming arrows into intervened variables), and the backdoor test \citep{pearl2009causality,peters2017elements,richardson2003markov}. We compile these operations into the selection rule itself. CALVER samples $K$ traces in a typed six-slot schema, checks each slot deterministically against the causal specification, selects the earliest trace attaining the maximum score, and returns its extracted answer (Figure~\ref{fig:overview}). The checker is training-free and independent of the sampling policy.

The same graph computations that power CALVER also support an alternative: construct one formal object and solve the query exactly. Our crossover study compares these two architectures. Direct solving is strongest when one reliable graph can be extracted. Candidate-wise verification remains effective when interpretations vary across traces and different candidates preserve different query-relevant relations. CALVER targets this intermediate regime and provides an empirical map of when each architecture is preferable.

Our contributions are:
\begin{enumerate}
\item \textbf{Formulation.} We formulate best-of-$K$ selection under answer multiplicity as aggregation over a decidable validity class. CALVER instantiates this idea with a six-component graph-dependent trace score for directed acyclic graphs (DAGs) and acyclic directed mixed graphs (ADMGs).
\item \textbf{Theory.} We prove soundness of a strict average-treatment-effect (ATE) certificate for thresholded ATE decisions and a validity-fragmentation theorem that explains why exact plurality can become inconsistent. We also characterize the query-local graphical information needed for constructed-graph verification.
\item \textbf{Experiments.} On the same inference traces, CALVER outperforms exact plurality, a set-aware medoid, model confidence, a learned reward model, and an LLM judge. The advantage grows through $K=32$ and transfers across graph families, graph representations, model families, an independently implemented do-calculus prover, and a formal logic task.
\end{enumerate}

%% file: Chapters/2_Motivation.tex
\section{Related Work}

CALVER intersects three lines of work. 

\paragraph{Inference-time selection.}
Self-consistency aggregates independently sampled traces by exact answer frequency \citep{wang2022self}. Soft self-consistency replaces exact counts with likelihood-based aggregation when viable outputs rarely match verbatim \citep{wang2024soft}. Semantic clustering and universal self-consistency similarly aggregate by meaning or use an LLM to select among free-form candidates \citep{kuhn2023semantic,farquhar2024detecting,chen2023universal}. Learned aggregators can review and synthesize candidate solutions after reinforcement learning with verifiable rewards \citep{zhao2025majority}. These alternatives relax exact string matching, but their evidence remains distributional agreement or a learned preference, and none checks satisfaction of a task axiom. Several distinct outputs can satisfy the same formal criterion, and consensus then penalizes correctness. CALVER applies the fixed executable predicate directly to each candidate, allowing a valid trace to be recognized without matching another candidate's answer string.

\paragraph{Learned and symbolic verification.}
Outcome verifiers and process reward models improve mathematical reasoning by ranking candidate solutions \citep{cobbe2021training,uesato2022solving,lightman2024let,wang2024math}. Symbolic systems use execution or deduction to check a model's formal output, including LEVER for code, SatLM for satisfiability, Logic-LM and LINC for logical inference, and AlphaGeometry for geometry \citep{ni2023lever,ye2023satlm,pan2023logic,olausson2023linc,trinh2024solving}. Analyses of best-of-$K$ with imperfect verifiers show that performance depends on candidate coverage and the verifier's full error profile \citep{stroebl2024inference,dorner2025roc,huang2025best}. CALVER brings executable verification to causal candidate selection through a typed trace contract and a target-label-free rule. Comparing fixed, learned, and prompted scorers on identical candidate traces tests whether causal axioms provide selection signal unavailable to generic verifiers.

\paragraph{Causal reasoning and causal proof checking.}
CLadder, Corr2Cause, CLEAR, and CausalGraph2LLM evaluate language models on causal questions and graph understanding \citep{jin2023cladder,jin2024can,chen2024clear,sheth2025causalgraph2llm}. DoVerifier checks whether a proposed causal expression is derivable from a graph using do-calculus and probability rules \citep{he2026uncovering}. Its published evaluation is target-directed: the expression to prove is supplied by the benchmark. We instead use it reference-free, treating each sampled expression as its own proof target and selecting the earliest provable candidate. This is candidate-wise selection with an independent prover, and it corroborates our principle without reusing our checker. A separate baseline, extract-then-solve, builds one global graph and solves the query exactly, bypassing candidate selection \citep{pearl2009causality,peters2017elements,richardson2003markov,shpitser2008complete}. Our crossover study measures when that route is preferable. CALVER differs from both by scoring a full structured trace: graph binding, strategy validity, recomputation, and answer consistency.

%% file: Chapters/3_Approach.tex
\section{Method}
\paragraph{Setting and trace contract.}
CALVER changes what best-of-$K$ aggregates. It scores each sampled trace against an executable causal specification. This section defines the objects, checks, and deterministic selection rule that implement that contrast.

Here, a problem contains a query $q$ and either a supplied source graph $G^\star$ or a text description generated from a source graph $G^\star$. In the text case, $G^\star$ is never exposed to any selector. It is retained outside selection for predeclared corpus construction checks, when applicable, and for final grading. The \emph{policy} (the language model under evaluation) produces $K$ independent traces $r_1,r_2, r_3,\ldots,r_K$ with extracted answers $a_i=a(r_i)$. The $K$ traces sampled for one problem form its candidate \emph{pool}. ATE problems also supply an observational distribution $P$ and a threshold $\tau$ (asking whether the effect exceeds $\tau$). 

Each output is judged by whether it belongs to the full set of graph-valid answers for the query. For graph $G$ and query $q$, let $\mathcal V(G,q)$ denote this \emph{validity class}. On a find-one-valid query, which asks for any single valid object, $a_i$ is correct when $a_i\in\mathcal V(G^\star,q)$. A dataset's listed answer is only one possible member.

\paragraph{Two modes.} Reflecting the two problem formats, verification runs on one of two graphs. In supplied-graph mode, every trace is scored on the source graph. In constructed-graph mode, used when structure must be read from text, trace $i$ emits its own graph $\widehat G_i$ and is scored on it, while $G^\star$ stays hidden until selection is complete. Write $G_i^{\mathrm{ver}}$ for the verification graph of trace $i$: $G_i^{\mathrm{ver}}=G^\star$ in the first mode and $G_i^{\mathrm{ver}}=\widehat G_i$ in the second. In constructed-graph mode, the score measures validity relative to $\widehat G_i$, whereas grading uses $G^\star$. CLEAR and the supplied-graph \texttt{bnlearn} rows use supplied-graph mode, and the graph-from-text causal studies use constructed-graph mode. The K\&K study follows the same candidate-local design, with a trace-specific propositional formalization checked by a truth-table engine in place of a causal graph.

In either mode, the scored object is the same. Each trace follows a typed schema, the trace contract, with six slots: \emph{graph}, \emph{query}, \emph{strategy}, \emph{derivation record}, \emph{computed result}, and \emph{answer}.  The schema makes intermediate claims machine-readable without constraining the surrounding rationale. The derivation-record check verifies provenance and format only. The graph, query, strategy, computed-result, and answer checks carry the certificate's semantic content.

Each task family has its own validity check, fixed before evaluation: backdoor adjustment, typed conditional independence, mediator witnesses, intervention reachability, or the numeric ATE check. Constructed-graph mode runs the same checks on $\widehat G_i$, the graph trace $i$ proposes.

\paragraph{Graph-dependent verification.}
Scoring is then deterministic: for trace $i$, the verifier returns bits
$b_{ij}=b_j(r_i;G_i^{\mathrm{ver}},q)\in\{0,1\}$ and score
$S_i=\sum_{j=1}^{6}b_{ij}$. The six bits audit the trace's inferential chain in order (Table~\ref{tab:criteria}). Missing or duplicate slots fail their corresponding components. We call $S_i=6$ the \emph{maximum deployed score}: all six checks pass relative to $G_i^{\mathrm{ver}}$. The strict ATE result below adds treatment exclusion and decision margin guards to certify the final threshold decision.

To make this more concrete, consider a backdoor query on graph $G$. Let $\De_G(X)$ be the strict descendants of $X$ and let $G_{\underline X}$ delete arrows leaving $X$. The deployed strategy bit tests
$X\perp_d Y\mid Z$ in $G_{\underline X}$ \citep{pearl2009causality,peters2017elements}. For an ADMG, the corresponding typed $m$-separation routine is used \citep{richardson2003markov}. For binary treatment and outcome, define the causal ATE
$\theta(P;X,Y)=\mathbb E_P[Y\mid\doop(X{=}1)]-\mathbb E_P[Y\mid\doop(X{=}0)]$.
Given an accepted adjustment set $Z$, the verifier evaluates the standard adjustment functional
\begin{equation}
\begin{aligned}
\psi(P;X,Y,Z)=\sum_z &\bigl[P(Y{=}1\mid X{=}1,Z{=}z)\\
&-P(Y{=}1\mid X{=}0,Z{=}z)\bigr]P(Z{=}z)
\end{aligned}
\label{eq:adjust}
\end{equation}
Let $\widehat\theta_i$ be the ATE value reported in trace $i$'s computed-result slot. The numerical bit independently recomputes $\psi(P;X,Y,Z)$ and requires
$|\widehat\theta_i-\psi(P;X,Y,Z)|\leq\varepsilon$, where $\varepsilon$ is declared before evaluation. The answer bit requires agreement with the threshold decision based on $\widehat\theta_i$.

For the soundness analysis, we reserve the term \emph{strict ATE certificate} for a maximum-score ATE trace that also satisfies
\begin{equation}
\begin{aligned}
Z\cap\bigl(\{X,Y\}\cup\De_G(X)\bigr)&=\varnothing,\\
X\perp_d Y\mid Z &\text{ in } G_{\underline X}
\end{aligned}
\label{eq:strictbackdoor}
\end{equation}
and whose recomputed effect satisfies
$|\psi(P;X,Y,Z)-\tau|>\varepsilon$. These guards use only the graph, observational distribution, query, and candidate trace.

The selector is therefore
$i^\star=\min\!\left(\argmax_{1\le i\le K}S_i\right)$ and returns $a_{i^\star}$. Score ties go to the earliest trace, while plurality ties go to the answer whose first occurrence is earliest. These rules are fixed before evaluation, and every comparison uses the same frozen candidates. A mechanical firewall removes correctness and benchmark-reference fields before any scorer runs. They are reattached only after the selected index is fixed. At $K=8$, symbolic verification adds about 7\% over the generation cost already paid by plurality, with each trace scored in 1--8\,ms on CPU.

\begin{table*}[htb]
\centering
\footnotesize
\renewcommand{\arraystretch}{1.12}
\begin{tabular}{@{}c l >{\raggedright\arraybackslash}p{0.30\textwidth} >{\raggedright\arraybackslash}p{0.36\textwidth}@{}}
\toprule
\textit{Bit} & \textit{Slot} & \textit{Question the check answers} & \textit{Executable test} \\
\midrule
$b_{i1}$ & \textsc{graph} & Does the trace's graph parse and bind to the instance? & Parse; bind to $G_i^{\mathrm{ver}}$ ($=G^\star$ supplied; $=\widehat G_i$ constructed) \\
$b_{i2}$ & \textsc{query} & Does the stated query bind to the asked one? & Parse; bind $(X,Y,\text{task})$ to the supplied query $q$ \\
\addlinespace
$b_{i3}$ & \textsc{strategy} & Is the proposed object graphically valid? & Backdoor: $X\perp_d Y\mid Z$ in $(G_i^{\mathrm{ver}})_{\underline X}$; ADMGs use typed $m$-separation \\
$b_{i4}$ & \textsc{derivation record}
& Does the derivation use the declared typed operation? & Provenance and format only \\
$b_{i5}$ & \textsc{computed result} & Does the reported value survive independent recomputation? & ATE: $|\widehat\theta_i-\psi(P;X,Y,Z)|\le\varepsilon$, with $\psi$ from \eqref{eq:adjust} and $\varepsilon$ declared in advance \\
$b_{i6}$ & \textsc{answer} & Does the final answer follow from the result? & Non-ATE: agreement with the recomputed result. ATE: agreement with $\mathbf{1}\{\widehat\theta_i>\tau\}$ \\
\addlinespace
\midrule
\multicolumn{2}{@{}l}{\textit{Score}} & \multicolumn{2}{l@{}}{$S_i=\sum_{j=1}^{6} b_{ij}$;\quad maximum deployed score $\;S_i=6$;\quad select $i^\star=\min\bigl(\argmax_i S_i\bigr)$} \\
\bottomrule
\end{tabular}
\caption{The six deterministic checks of the trace contract. Each bit $b_{ij}\in\{0,1\}$ is computed from the trace and its verification graph alone. No target label is consulted, and missing or duplicate slots fail their components.}
\label{tab:criteria}
\end{table*}

%% file: Chapters/3a_Theory.tex
\section{Theory}
The theory addresses three questions: when a maximum-score ATE trace certifies a correct threshold decision; why exact plurality can fail when valid probability mass is divided across answer strings; and which graph relations must survive text-based reconstruction. Full proofs are provided in the Supplementary document.

\begin{theorem}[Strict ATE certificate soundness]
\label{thm:sound}
Assume that the supplied DAG $G^\star$ and observational distribution $P$ are correct and that $P$ is positive. Then any binary ATE trace satisfying the strict ATE certificate defined in the Method has the correct final threshold decision. The verifier does not use a target label.
\end{theorem}

\noindent\emph{Proof sketch.}
The treatment-exclusion and $d$-separation conditions in the strict ATE certificate make $Z$ a valid backdoor adjustment set. Positivity and the adjustment formula therefore give
$\psi(P;X,Y,Z)=\theta(P;X,Y)$. The maximum deployed score ensures
$|\widehat\theta_i-\psi(P;X,Y,Z)|\leq\varepsilon$ and requires the final answer to agree with the threshold decision based on $\widehat\theta_i$. The strict decision margin keeps $\widehat\theta_i$ and the true ATE on the same side of $\tau$. The reported qualitative decision is therefore correct.

\begin{theorem}[Validity fragmentation]
\label{thm:frag}
Let $A_1,\ldots,A_K$ be i.i.d.\ answers from a finite mass function $\mu$, with validity class $\mathcal V$. Suppose one invalid answer $w$ exceeds the mass of every \emph{individual} supported valid answer by at least $\delta_\mu>0$:
$\mu(w)\geq\mu(v)+\delta_\mu$ for every supported $v\in\mathcal V$. Then, for any deterministic plurality tie rule:
\begin{equation}
\Pr\{\operatorname{plurality}(A_{1:K})\in\mathcal V\}
\leq |\mathcal V_+|\exp(-K\delta_\mu^2/2)
\end{equation}
where $\mathcal V_+=\{v\in\mathcal V:\mu(v)>0\}$.
\end{theorem}

\noindent\emph{Proof sketch.}
For each supported valid answer $v$, compare its empirical count with the count of $w$. Hoeffding's inequality bounds the probability that $v$ ties or exceeds $w$ by $\exp(-K\delta_\mu^2/2)$. A union bound over supported valid answers gives the result.

A concrete example makes this reversal explicit. Let valid answers carry total mass $\rho=0.6$, divided evenly among $M=4$ answer strings, and let one invalid answer carry the remaining mass $0.4$. Each valid string then has mass $0.15$, so exact plurality converges to the invalid answer even though a sampled answer is valid with probability $0.6$. More generally, this reversal occurs whenever $M>\rho/(1-\rho)$. By contrast, an ideal validity-separating verifier succeeds whenever at least one valid candidate appears, with probability $1-(1-\rho)^K$. This idealized comparison isolates the aggregation mechanism.

\paragraph{Why imperfect graph recovery can still suffice.}
Constructed-graph verification need not recover every edge. It must preserve the relations that determine the queried predicate. For a proposed backdoor set, those relations are treatment exclusion, descendant status, and $d$-separation in the mutilated graph. This creates a distinction: an error on a query-irrelevant edge can leave the decision unchanged, whereas one omitted confounding edge can reverse it even when global edge F1 is close to one. The Supplementary document formalizes this distinction.

%% file: Chapters/4_Empirical_study.tex
\section{Experiments}

\paragraph{Protocol and evaluation questions.}
The experiments ask four questions: whether validity selection beats strong non-oracle scorers on identical pools, whether gains scale with $K$, whether they transfer to external structures and another formal domain, and what changes when the graph must be constructed from text. The eight policies in Table~\ref{tab:policies} span Qwen2.5 at 7B, 14B, and 32B parameters \citep{yang2024qwen25}, 4-bit NF4 variants at 7B and 32B \citep{dettmers2023qlora}, and Mistral NeMo 12B \citep{mistral2024nemo}. The 7B SFT policy is trained on the six-slot trace contract, and the verifier-GRPO policy uses the same contract with a verifier-derived reward. The Supplementary document summarizes the training and evaluation configuration; exact model and adapter revisions, NF4 settings, reward definitions, prompts, seeds, and run manifests will accompany the public Code and Data Package.
Unless stated otherwise, sampling uses temperature 0.8 and $K=8$, and intervals are paired bootstraps clustered by problem. Reported gains are calculated before displayed estimates are rounded. External graphs and logic puzzles are zero-shot with respect to the training generator.

These questions need a testbed where multiplicity is genuine, and validity is computable. Our primary benchmark is therefore the typed clean core of CLEAR's find-one-valid tasks \citep{chen2024clear}: the 126 of the 480 CLEAR find-one-valid items whose published task admits multiple valid objects and for which a graph-class-compatible predicate is implemented. This inclusion rule is fixed from the published task type and graph semantics, before any candidate is generated and therefore before any model output, correctness outcome, or selector score exists. A selected object is accepted whenever it satisfies the predicate, including when it differs from the dataset's listed example. Supplied-graph CLEAR isolates the selection problem under exact graph semantics. On supplied-graph CLEAR the validity check and the grading criterion apply the same predicate by construction. The graph-from-text, DoVerifier, K\&K, and exact-ground-truth ATE studies, where verifier and evaluator are distinct, carry that separation. Accuracy is measured at the problem-seed level unless a result is explicitly labeled by trace or item.

\paragraph{Many-satisfier queries across policies.}
We begin on the primary benchmark: the fixed checker raises the selected-answer point estimate for all eight policies on the common set of 126 problems (Table~\ref{tab:policies}). The set-aware Jaccard medoid chooses the candidate set with the largest average overlap with the pool. Relative to this stronger voting baseline, seven of eight confidence intervals exclude zero. The Qwen-7B base estimate is positive but its interval includes zero. Policies trained or prompted to emit the trace contract more reliably provide more usable verifier signal, while the decision procedure remains unchanged. In the audited clean-core subset, 11 of the 21 graph-valid answers CALVER selected (52.4\%) differ from the dataset's listed example while satisfying the requested predicate. Exact-match voting treats these as separate modes, and the verifier recognizes their shared validity.

\begin{table*}[t]
\centering
\small
\setlength{\tabcolsep}{6pt}
\begin{tabular}{@{}lrrrrrr@{}}
\toprule
Policy & $n$ & First & Exact plurality & Set medoid & CALVER & $\Delta$ vs.\ medoid (pp; 95\% CI) \\
\midrule
Qwen 7B base          & 360 & 23.6 & 31.4 & 32.2 & 36.7 & $+4.4$~~[$-0.8$, 9.7] \\
Qwen 7B SFT           & 377 & 23.6 & 28.9 & 31.6 & 43.8 & $+12.2$~[6.9, 17.8] \\
Qwen 7B verifier-GRPO & 374 & 27.5 & 33.2 & 35.3 & 46.0 & $+10.7$~[5.6, 15.8] \\
Qwen 7B NF4           & 372 & 14.0 & 17.7 & 20.4 & 38.4 & $+18.0$~[13.3, 23.0] \\
Mistral NeMo 12B      & 378 & 21.7 & 27.5 & 28.3 & 47.4 & $+19.0$~[14.6, 24.1] \\
Qwen 14B              & 378 & 41.8 & 45.8 & 44.4 & 68.5 & $+24.1$~[18.8, 29.6] \\
Qwen 32B              & 378 & 47.1 & 48.7 & 47.4 & 62.4 & $+15.1$~[10.1, 20.6] \\
Qwen 32B NF4          & 378 & 44.4 & 51.6 & 52.1 & 61.9 & $+9.8$~~[5.6, 14.6] \\
\bottomrule
\end{tabular}
\caption{Selection accuracy (\%) on the CLEAR clean core at $K=8$. All policies use the
same 126 source problems; $n$ is the number of available problem--seed units. The final
column reports CALVER's percentage-point gain over the set medoid. Brackets are paired
95\% confidence intervals, bootstrapped over resamples clustered by source problem.}
\label{tab:policies}
\end{table*}

\paragraph{Where does CALVER repair plurality's errors?}
These gains arise exactly where the fragmentation model would predict. The predicted spoiler effect is common in the observed pools: among 576 clean-core pools containing at least one valid candidate, an invalid answer is the unique plurality in 226 (39\%). CALVER selects a valid object in 68.6\% of those cases. Across all 576 pools, CALVER repairs 197 plurality errors and changes 54 correct selections to errors, a ratio of 3.6 repairs per new error with a net gain of 24.8 percentage points (full decomposition in the Supplementary document).

\paragraph{Does the advantage survive matched pools and larger $K$?}
We compare every selector on the 1,111 problem-seed units for which all scores are present. The generic comparators are Skywork Reward V2 8B \citep{liu2025skywork}, a reference-free LLM judge that receives the same graph, query, and trace, and the policy's own binary confidence score. Each scalar scorer selects the earliest candidate attaining its maximum score. Exact plurality and the set medoid use their predeclared answer-level tie rules. 

CALVER reaches 42.1\% (Figure~\ref{fig:mainresults}); the closest generic comparator reaches 30.5\%. Paired problem-clustered gains are 11.6 percentage points over the reward model [8.2, 15.0], 14.8 over the judge [11.4, 18.2], 12.1 over model confidence [8.6, 15.6], and 11.3 over exact plurality [7.8, 14.9]. A structure-only selector that retains parsing, binding, and format checks while removing semantic validity reaches 23.5\%. The 18.6-point difference [15.6, 21.7] attributes the gain to the executable causal predicate: checking format alone confers no advantage. Scaling the judge does not recover the signal as a Qwen2.5-72B judge given the same graph ties plurality (+0.4 [-2.0, 3.0]).

As an independent formal check, candidate-target DoVerifier raises first-sample accuracy from 49.6\% to 80.0\% on 240 gated causal-identification items without receiving a reference expression. This is an independent implementation of a candidate-wise selection principle. Results on each query subset (many-satisfier and per-network) are reported in the Supplementary document.

\begin{figure*}[htb]
\centering
\includegraphics[width=\textwidth]{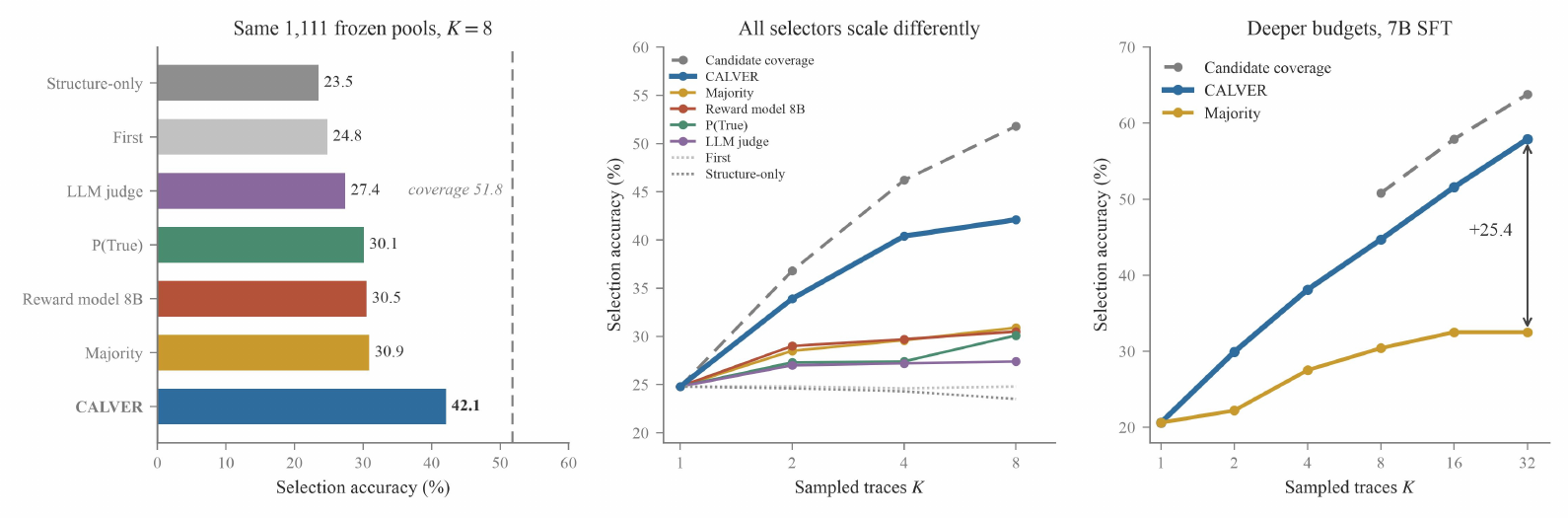}
\caption{\textbf{Left:} non-oracle selectors on the 1,111 CLEAR units for which every scorer is available. \textbf{Center:} prefix scaling on the same frozen eight-candidate pools. \textbf{Right:} deeper scaling on the Qwen-7B SFT clean core. The dashed gray curve labelled candidate coverage is the fraction of pools with at least one correct candidate, computed from held-out labels after generation. It is a diagnostic ceiling and is unavailable to all selectors.}
\label{fig:mainresults}
\end{figure*}

Returning to the frozen CLEAR pools, the gap grows with the sampling budget over the tested prefixes. On the SFT policy, CALVER rises from 20.6\% at $K=1$ to 29.9, 38.1, 44.7, 51.6, and 57.9\% for $K=2,4,8,16,32$. Exact plurality reaches 32.5\% and does not improve from $K=16$ to $K=32$. The gap therefore grows from 7.7 percentage points at $K=2$ to 25.4 points at $K=32$ [18.3, 32.3]. Every point is a prefix of one frozen $K=32$ pool, generated separately from the $K=8$ pool in Table~\ref{tab:policies}. The verifier-GRPO and unadapted-base policies show the same qualitative pattern, with $K=32$ gaps of 23.8 and 10.3 points. This behavior is consistent with the fragmentation mechanism in Theorem~\ref{thm:frag}.

A process reward model distilled from the checker's scores does not reproduce its selection. On the frozen ATE pools, it selects no better than plurality and does not improve with $K$, while the symbolic checker rises toward the candidate-coverage ceiling (Supplementary document).

\paragraph{External transfer.}
Table~\ref{tab:transfer} separates supplied-graph transfer, graph-from-text transfer, and a formal logic task. The Bayesian-network structures come from the \texttt{bnlearn} repository \citep{scutari2022bnrepository}. The node-renaming audit and prose-recovery measurements limit reliance on memorized network names. CausalGraph2LLM contributes mediator and intervention tasks across textual, JSON, adjacency, GraphML, and Graphviz encodings \citep{sheth2025causalgraph2llm}. The K\&K benchmark provides dynamically generated logic puzzles \citep{xie2025memorization}. The policy emits a propositional formalization, and a truth-table engine instantiates the same validity-selection principle.

\begin{table}[t]
\centering
\footnotesize
\setlength{\tabcolsep}{1.5pt}
\begin{tabular}{
  >{\hangindent=1em\hangafter=1\relax}p{.39\columnwidth}
  rrr
}
\toprule
Evaluation & Plurality & CALVER & \shortstack{Gain\\(pp)} \\
\midrule
CLEAR clean core, SFT ($n=377$) & 28.9 & 43.8 & $+14.9$ \\
\texttt{bnlearn}, supplied DAG, 10 nets ($n=410$) & 39.8 & 56.8 & $+17.1$ \\
L1 edge-stated text, 9 nets ($n=132$) & 50.8 & 75.0 & $+24.2$ \\
L2 mechanism-explicit text, 8 nets ($n=246$) & 31.3 & 45.5 & $+14.2$ \\
L3 naturalistic text, 8 nets ($n=273$) & 33.3 & 50.9 & $+17.6$ \\
CausalGraph2LLM, five encodings ($n=4{,}800$) & 55.6 & 58.3 & $+2.8$ \\
K\&K, 3 people ($n=180$) & 50 & 77 & $+27$ \\
K\&K, 4 people ($n=180$) & 26 & 63 & $+37$ \\
K\&K, 5 people ($n=180$) & 22 & 38 & $+16$ \\
\bottomrule
\end{tabular}
\caption{Non-oracle selector accuracy (\%) on independently sourced benchmarks and structures. The symbolic checker is matched to the task: causal graph predicates for the causal rows and truth-table consistency for K\&K. $n$ counts problem--seed units, and L2/L3 use three seeds. Gains are CALVER minus plurality in percentage points.}
\label{tab:transfer}
\end{table}

The \texttt{bnlearn} ladder asks what happens to selection when the same repository networks must be reconstructed from increasingly indirect prose. The supplied-graph rows anchor the comparison: the pooled gain is 17.1 percentage points [13.4, 20.7]. Rendering the same structures as prose forces the trace to reconstruct a graph before proposing an adjustment set, and lets us vary how directly the structure is stated. At L1, the scenario names each edge, and extraction is near transcription. Exact graph recovery is 83\% and the gain is 24.2 percentage points [17.4, 31.8]. At L2 the text describes mechanisms explicitly, and at L3 it is naturalistic narrative with distractors. Exact recovery then falls to 37\% and 34\%, yet the gain remains 14.2 percentage points [8.5, 20.7] and 17.6 percentage points [11.4, 24.2]. Thus, on this round-trip-filtered \texttt{bnlearn} corpus, CALVER continues to improve over plurality even when exact whole-graph recovery falls below 40\%. The query-local analysis offers a mechanism for this pattern: a trace may preserve the relations needed for a decision without recovering every edge.

The round-trip faithfulness filter is fixed before evaluation and retains an item only when re-extracting the prose reproduces the source subgraph. It uses neither model answers nor selector scores, and every selector sees the same admitted corpus. The filter retains nine networks at L1 and eight at L2/L3. On the latter two levels, a structure-only control reaches 29.3\% and 30.8\%, below the corresponding plurality baselines.

The final rows of Table~\ref{tab:transfer} leave causal graphs entirely, turning to K\&K puzzles: from each speaker's statements, decide who always lies and who always tells the truth. This analysis shows that the selection mechanism is not specific to causal axioms. Substituting the domain's own decision procedure, truth-table consistency in place of the graphical predicates, reproduces the gain over plurality at all three puzzle sizes, with no task-specific fine-tuning.

\paragraph{When should one extract a single graph instead?}
When transcription is reliable, a graph stated in text can instead be extracted once and passed to an exact solver. We compare this route with candidate-wise verification on a separate construction suite of eight graph motifs rendered as edge statements (L1), mechanism descriptions (L2), and naturalistic narratives with distractors (L3). Each level contains 240 problem-seed units.

\begin{figure}[t]
\centering
\includegraphics[width=\columnwidth]{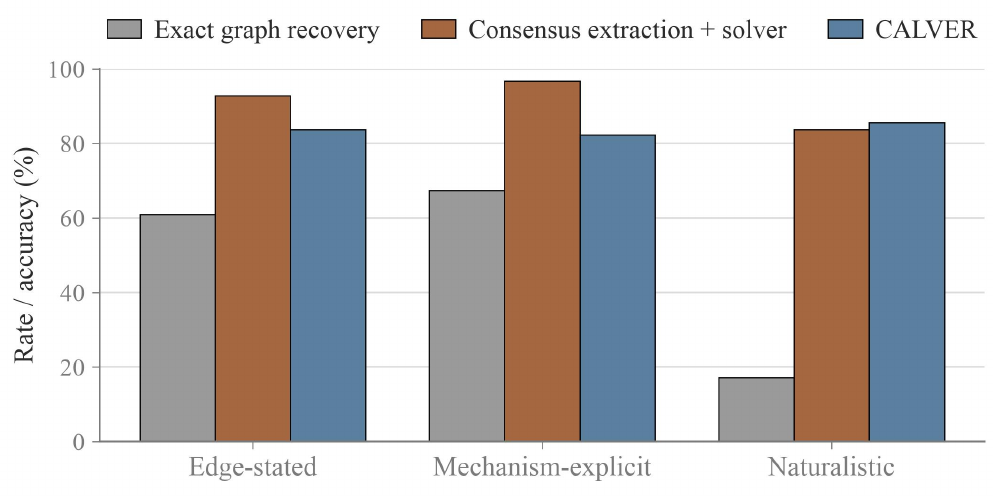}
\caption{Exact solving after consensus graph extraction versus candidate-wise verification
on the 153 executable units per presentation level.}
\label{fig:crossover}
\end{figure}

Across all 240 units per level, exact graph recovery falls from 39\% at L1 to 6\% at L3, while CALVER's gain over plurality remains between 50 and 52 percentage points. On the common 153-unit executable stratum, where recovery is necessarily higher because parseability is a precondition, extraction followed by exact solving reaches 92.8\% and 96.7\% at L1 and L2, compared with 83.7\% and 82.4\% for candidate-wise verification. At L3, the observed accuracies are similar: 83.7\% for extract-then-solve and 85.6\% for candidate-wise verification (Figure~\ref{fig:crossover}). A practical rule follows: when transcription is reliable, extract one graph and solve it directly; when readings of the text diverge, no single extraction can be trusted, and verifying every candidate is the better option.

\paragraph{Mechanism and ATE audits.}
Table~\ref{tab:mechanism-audits} isolates the source of the selection signal. Removing semantic validity eliminates the gain, while increasingly severe graph corruption produces an ordered reduction. Tie-rule and node-renaming audits show that the result is not explained by candidate ordering or lexical identity. Randomized property tests additionally compare the graph primitives with an independent NetworkX implementation \citep{hagberg2008exploring}. 

\begin{table}[htb]
\centering
\footnotesize
\renewcommand{\arraystretch}{1.1}
\begin{tabular}{@{}>{\raggedright\arraybackslash}p{.24\columnwidth} >{\raggedright\arraybackslash}p{.40\columnwidth} >{\raggedright\arraybackslash}p{.26\columnwidth}@{}}
\toprule
Control & Result & Conclusion \\
\midrule
\textit{(1)} Validity check removed (structure-only) & Falls below first sample and plurality (CLEAR; \texttt{bnlearn} L2/L3) & Format alone confers no gain \\
\textit{(2)} Ties resolved to worst candidate & $\le 0.6$\,pp shift on seven of eight policies; exception: Qwen 32B NF4 & Not a tie artifact \\
\textit{(3)} Consistent node renaming & All 19{,}176 scores unchanged & No lexical leakage \\
\textit{(4)} Graph corruption (10\% del.\ / 10\% rev.\ / 25\% mixed / full) & Gain vs.\ plurality: $14.7\!\to\!13.7/11.5/10.4/6.6$\,pp & Structure drives selection; degrades monotonically \\
\textit{(5)} Schema adaptation removed & Gain on control subset: 5.0\,pp base; 10.4--18.5\,pp adapted & Adapter mediates usable signal \\
\bottomrule
\end{tabular}
\caption{Mechanism controls. Each row removes or perturbs one ingredient of the selection signal on otherwise identical configurations.}
\label{tab:mechanism-audits}
\end{table}

The strict ATE audit applies the treatment-exclusion and decision-margin guards in addition to the deployed six-component score. Among 2,376 evaluated traces, 585 attain the maximum deployed score. The 51 incorrect maximum-score traces all violate at least one strict guard: 45 lie within the declared decision margin and six place the treatment in their proposed adjustment set. No incorrect trace passes the strict ATE certificate. On the 249-problem declared-margin stratum, CALVER gains 13.7 percentage points over exact plurality [7.2, 20.1]. Additional calibration and graph-class results are in the Supplementary document.

%% file: Chapters/5_Conclusion.tex
\section{Limitations \& Conclusion}

CALVER changes the unit of aggregation from answer-string frequency to evidence that a candidate belongs to a task-defined validity class. CALVER's gain concentrates where correct answers fragment. Among the 576 pools that contain at least one valid candidate, it repairs 197 plurality errors and introduces 54, for a net gain of 24.8 percentage points. Across eight policies, published graph families, graph-from-text settings, and a truth-table logic task, the same candidate-wise principle raises the selected-answer point estimate. The widening observed gap through $K=32$ is consistent with the validity-fragmentation mechanism in Theorem~\ref{thm:frag}.

CALVER applies wherever the validity predicate can be executed, covering the graphical criteria used here and the truth-table check used for Knights and Knaves. Its advantage is regime-bound. When answers are near-unique, best-of-$K$ already suffices, and when text-to-graph extraction is reliable, extracting once and solving directly is the better tool. CALVER is the right choice in between, where many answers are valid and no single extraction is dependable.